\documentclass{article} 
\usepackage{iclr2021_conference_for_arXiv,times}
\iclrfinalcopy

\usepackage{amsmath,amssymb}
\usepackage{amsbsy}
\usepackage{bm}
\usepackage{amsthm}
\usepackage{algorithm,algorithmic}
\usepackage{xspace}
\usepackage{comment}

\newtheorem{Theorem}{Theorem}
\newtheorem{Lemma}{Lemma}

\newtheorem{Assumption}{Assumption}

\newtheorem{Proposition}{Proposition}
\newtheorem{Remark}{Remark}
\newtheorem{Corollary}{Corollary}

\usepackage{enumitem}
\newlist{assumenum}{enumerate}{1} 
\setlist[assumenum]{label={\rm (\roman*)}, ref=(\roman*)}

\newcommand{\Eqref}[1]{Eq. \eqref{#1}}

\newcommand{\calA}{{\mathcal A}}

\newcommand{\calE}{{\mathcal E}}
\newcommand{\calF}{{\mathcal F}}

\newcommand{\calH}{{\mathcal H}}

\newcommand{\calL}{{\mathcal L}}

\newcommand{\fhat}{\widehat{f}}

\newcommand{\conv}{\mathrm{conv}}

\newcommand{\Px}{P_{X}}

\newcommand{\LPiPx}{L_2(\Px)}

\newcommand{\Real}{\mathbb{R}}
\newcommand{\Natural}{\mathbb{N}}

\newcommand{\EE}{\mathrm{E}}
\newcommand{\dd}{\mathrm{d}}
\newcommand{\fstar}{f^{\ast}}

\newcommand{\ftrue}{f^{\mathrm{o}}}

\def\I<#1>{\left\langle #1 \right\rangle}
\def\i<#1>{\left\langle #1 \right\rangle}

\newcommand{\supp}{\mathrm{supp}}

\newcommand{\calLhat}{\widehat{\calL}}

\newcommand{\cH}{\mathcal{H}}

\newcommand{\norm}[1]{\|#1\|}

\newcommand{\Wstar}{W^*}

\newcommand{\nutil}{\tilde{\nu}_{\beta}}

\newcommand{\alphatil}{\tilde{\alpha}}

\newcommand{\epsilonstar}{{\epsilon^*}}

\newcommand{\wm}[1]{w_{#1}}
\newcommand{\wmb}[1]{\bar{w}_{#1}}

\newcommand{\aone}{\alpha_1}
\newcommand{\atwo}{\alpha_2}

\newcommand{\mn}{{m_n}}

\newcommand{\convbar}{\overline{\conv}}

\newcommand{\Rlin}{R_{\mathrm{lin}}}
\newcommand{\Rnn}{R_{\mathrm{NN}}}

\allowdisplaybreaks
\usepackage{url}

\usepackage{hyperref}
\hypersetup{
  colorlinks,
  linkcolor={red!100!black},
  citecolor={blue!100!black},
  urlcolor={blue!80!black}
}

\newcommand{\Id}{\mathrm{I}}

\newcommand{\Rbar}{\bar{R}}

\newcommand{\GLD}{{NGD}\xspace}

\usepackage{xr}
\usepackage[textwidth=2.5cm, textsize=tiny]{todonotes}

\newcommand{\shorttitle}{Benefit of deep learning with non-convex noisy gradient descent}

\title{Benefit of deep learning with non-convex noisy gradient descent: Provable excess risk bound and superiority to kernel methods}

\author{Taiji Suzuki \\
Graduate School of Information Science and Technology, The University of Tokyo, Japan \\
Center for Advanced Intelligence Project, RIKEN, Japan\\
\texttt{taiji@mist.i.u-tokyo.ac.jp} \\ 
\AND
Shunta Akiyama \\
Graduate School of Information Science and Technology, The University of Tokyo, Japan \\
\texttt{akiyama@mist.i.u-tokyo.ac.jp}
}

\begin{document}

\maketitle


\begin{abstract}
Establishing a theoretical analysis that explains why deep learning can outperform shallow learning such as kernel methods is one of the biggest issues in the deep learning literature.
Towards answering this question, we evaluate excess risk of a deep learning estimator trained by a noisy gradient descent with ridge regularization on a mildly overparameterized neural network, 
and discuss its superiority to a class of linear estimators that includes neural tangent kernel approach, random feature model, other kernel methods, $k$-NN estimator and so on.
We consider a teacher-student regression model, 
and eventually show that {\it any} linear estimator can be outperformed by deep learning in a sense of the minimax optimal rate especially for a high dimension setting. 
The obtained excess bounds are so-called fast learning rate which is faster than $O(1/\sqrt{n})$ that is obtained by usual Rademacher complexity analysis.
This discrepancy is induced by the non-convex geometry of the model and the noisy gradient descent used for neural network training provably reaches a near global optimal solution even though the loss landscape is highly non-convex. 
Although the noisy gradient descent does not employ any explicit or implicit sparsity inducing regularization, it shows a preferable generalization performance that dominates linear estimators.
\end{abstract}

\section{Introduction}


In the deep learning theory literature, clarifying the mechanism by which deep learning can outperform shallow approaches has been gathering most attention for a long time.
In particular, it is quite important to show that a tractable algorithm for deep learning can provably achieve a better generalization performance than shallow methods. 
Towards that goal, we study the rate of convergence of {\it excess risk} of both deep and shallow methods
in a setting of a nonparametric regression problem.
One of the difficulties to show generalization ability of deep learning with certain optimization methods is that 
the solution is likely to be stacked in a bad local minimum, which prevents us to show its preferable performances.
Recent studies tackled this problem by considering optimization on overparameterized networks as in neural tangent kernel (NTK) \citep{jacot2018neural,du2019gradient} and mean field analysis \citep{nitanda2017prticle,chizat2018note,NIPS:Rotskoff&Vanden-Eijnden:2018,rotskoff2019trainability,MeiE7665,pmlr-v99-mei19a}, or analyzing the noisy gradient descent such as stochastic gradient Langevin dynamics (SGLD) \citep{Welling_Teh11,Raginsky_Rakhlin_Telgarsky2017,NIPS2018_8175}.

The NTK analysis deals with a relatively large scale initialization so that the model is well approximated by the tangent space at the initial solution, and eventually, all analyses can be reduced to those of kernel methods \citep{jacot2018neural,du2018gradient,allen2019convergence,du2019gradient,arora2019fine,cao2019generalization,zou2020gradient}.
Although this regime is useful to show its global convergence, 
the obtained estimator looses large advantage of deep learning approaches because 
the estimation ability is reduced to the corresponding kernel methods. 
To overcome this issue, there are several ``beyond-kernel'' type analyses. 
For example, \cite{AllenZhu:NIPS2019:BeyondKernel,allen2020backward} showed benefit of depth by analyzing ResNet type networks.
\cite{pmlr-v125-li20a} showed global optimality of gradient descent by reducing the optimization problem to a tensor decomposition problem for a specific regression problem, 
and showed the ``ideal'' estimator on a linear model has worse dependency on the input dimensionality.
\cite{Bai2020Beyond} considered a second order Taylor expansion and showed that the sample complexity of deep approaches has better dependency on the input dimensionality than kernel methods.
\cite{chen2020understanding} also derived a similar conclusion by considering a hierarchical representation.
The analyses mentioned above actually show some superiority of deep learning, but all of these bounds are essentially $\Omega(1/\sqrt{n})$ where $n$ is the sample size, which is not optimal for regression problems with squared loss \citep{FCM:Caponetto+Vito:2007}. 
The reason why only such a sub-optimal rate is considered is that the target of their analyses is mostly the Rademacher complexity of the set in which estimators exist for bounding the generalization gap. 
However, to derive a tight {\it excess risk} bound instead of the generalization gap, we need to evaluate so called {\it local Rademacher complexity} \citep{IEEEIT:Mendelson:2002,LocalRademacher,Koltchinskii}
(see \Eqref{eq:ExcessRiskDefinition} for the definition of excess risk). 
Moreover, some of the existing analyses should change the target function class as the sample size $n$ increases, 
for example, the input dimensionality is increased against the sample size, 
which makes it difficult to see how the rate of convergence is affected by the choice of estimators.

Another promising approach is the mean field analysis. 
There are also some work that showed superiority of deep learning against kernel methods. 
\cite{ghorbani2019linearized} showed that, when the dimensionality $d$ of input is polynomially increasing with respect to $n$, the kernel methods is outperformed by neural network approaches. 
Although the situation of increasing $d$ explains well the modern high dimensional situations, 
this setting blurs the rate of convergence.
Actually, we can show the superiority of deep learning even in a {\it fixed dimension} setting.

There are several studies about approximation abilities of deep and shallow models.
\cite{ghorbani2020neural} showed adaptivity of kernel methods to the intrinsic dimensionality in terms of approximation error and discuss difference between deep and kernel methods. 
\cite{NIPS2019_8886} showed that the random feature method requires exponentially large number of nodes against the input dimension to obtain a good approximation for a single neuron target function. 
These are only for approximation errors and 
estimation errors are not compared.



Recently, the superiority of deep learning against kernel methods has been discussed also in the nonparametric statistics literature 
where the minimax optimality of deep learning in terms of excess risk is shown.
Especially, it is shown that deep learning achieves better rate of convergence than {\it linear estimators} in several settings
\citep{AoS:Schmidt-Hieber:2020,suzuki2018adaptivity,pmlr-v89-imaizumi19a,suzuki2019deepIntrinsicDim,HAYAKAWA2020343}.
Here, the linear estimators are a general class of estimators that includes kernel ridge regression, $k$-NN regression and Nadaraya-Watson estimator.
Although these analyses give clear statistical characterization on estimation ability of deep learning,
they are not compatible with tractable optimization algorithms.

In this paper, we give a theoretical analysis that unifies these analyses and
shows the superiority of a deep learning method trained by a tractable noisy gradient descent algorithm.  
We evaluate the excess risks of the deep learning approach and linear estimators in a nonparametric regression setting, 
and show that the minimax optimal convergence rate of the linear estimators can be dominated by the noisy gradient descent on neural networks. 
In our analysis, the model is fixed and no explicit sparse regularization is employed.
Our contributions can be summarized as follows:
{\vspace{-0.05cm}
\begin{itemize}[leftmargin=*,topsep=0.5mm,itemsep=0.2mm] 
\item A refined analysis of excess risks for a fixed model with a fixed input dimension is given to compare deep and shallow estimators. 
Although several studies pointed out the curse of dimensionality is a key factor that separates shallow and deep approaches, we point out that such a separation appears in a rather low dimensional setting,
and more importantly, the {\it non-convexity} of the model essentially makes the two regimes different.
\item A lower bound of the excess risk which is valid for {\it any} linear estimator is derived.
The analysis is considerably general because the class of linear estimators includes kernel ridge regression with any kernel and thus it also includes estimators in the NTK regime. 
\item 
All derived convergence rate is a fast learning rate that is faster than $O(1/\sqrt{n})$.
We show that simple noisy gradient descent on a sufficiently wide two-layer neural network achieves a fast learning rate by using a fact that the solution converges to a Bayes estimator with a Gaussian process prior, and the derived convergence rate can be faster than that of linear estimators. 
This is much different from such existing work that compared only coefficients with the same rate of convergence with respect to the sample size $n$.
\end{itemize}
\vspace{-0.05cm}
}

\paragraph{Other related work} \cite{JMLR:v18:14-546} analyzed the model capacity of neural networks and its corresponding reproducing kernel Hilbert space (RKHS), and showed that the RKHS is much larger than the neural network model. 
However, separation of the estimation abilities between shallow and deep is not proven. 
Moreover, the analyzed algorithm is basically the Frank-Wolfe type method which is not typically used in practical deep learning. The same technique is also employed by \cite{barron1993universal}.
The Frank-Wolfe algorithm is a kind of sparsity inducing algorithm that is effective for estimating a function in a model with an $L_1$-norm constraint.
It has been shown that explicit or implicit sparse regularization such as $L_1$-regularization is beneficial to obtain better performances of deep learning under certain situations \citep{chizat2020implicit,chizat2019sparse,gunasekar2018implicit,pmlr-v125-woodworth20a,klusowski2016risk}.
For example, \cite{weinan2019comparative,ma2019priori} showed that the approximation error of a linear model suffers from the curse of dimensionality in a setting where the target function is in the Barron class  \citep{barron1993universal}, and showed an $L_1$-type regularization avoids the curse of dimensionality.
However, our analysis goes in a different direction where a sparse regularization is not required.

\section{Problem setting and model}

In this section, we give the problem setting and notations that will be used in the theoretical analysis.
We consider the standard nonparametric regression problem where 
data are generated from the following model for an unknown true function $\ftrue:\Real^d \to \Real$:  
\begin{align}\label{eq:ObsModel}
y_i = \ftrue(x_i) + \epsilon_i~~~(i=1,\dots,n),
\end{align}
where $x_i$ is independently identically distributed from $\Px$ whose support is included in $\Omega = [0,1]^d$,
and $\epsilon_i$ is an observation noise that is independent of $x_i$ and satisfies $\EE[\epsilon_i] = 0$ and $\epsilon_i \in [-U,U]$ almost surely. 
The $n$ i.i.d. observations are denoted by $D_n = (x_i,y_i)_{i=1}^n$.
We want to estimate the true function $\ftrue$ through the training data $D_n$.
To achieve this purpose, we employ the squared loss $\ell(y,f(x)) = (y-f(x))^2$ and accordingly we 
define the expected and empirical risks as 
$\calL(f) := \EE_{Y,X}[\ell(Y,f(X))]$ and $\calLhat(f) := \frac{1}{n} \sum_{i=1}^n \ell(y_i,f(x_i))$ respectively. 
Throughout this paper, we are interested in the {\it excess (expected) risk} of an estimator $\fhat$ defined by 
\begin{align}\label{eq:ExcessRiskDefinition}
\text{(Excess risk)} & &\calL(\fhat) - \inf_{f:\text{measurable}} \calL(f). &  &  
\end{align}
Since the loss function $\ell$ is the squared loss, the infimum of $\inf_{f:\text{measurable}} \calL(f)$ is achieved by $\ftrue$:  $\inf_{f:\text{measurable}} \calL(f) = \calL(\ftrue)$.
The population $L_2(\Px)$-norm is denoted by $\|f\|_{\LPiPx} := \sqrt{\EE_{X\sim \Px}[f(X)^2]}$
and the sup-norm on the support of $\Px$
is denoted by $\|f\|_\infty := \sup_{x \in \supp(\Px)}|f(x)|$.
We can easily check that for an estimator $\fhat$, the $L_2$-distance $\|\fhat - \ftrue\|_{\LPiPx}^2$ between the estimator $\fhat$ and the true function $\ftrue$ is identical to the excess risk: $\calL(\fhat) - \calL(\ftrue) = \|\fhat - \ftrue\|_{\LPiPx}^2$.
Note that the excess risk is different from the generalization gap $\calL(\fhat) - \calLhat(\fhat)$.
Indeed, the generalization gap typically converges with the rate of $O(1/\sqrt{n})$ which is optimal in a typical setting  \citep{mohri2012foundations}.
On the other hand, the excess risk can be faster than $O(1/\sqrt{n})$, which is known as a {\it fast learning rate}
\citep{IEEEIT:Mendelson:2002,LocalRademacher,Koltchinskii,gine2006concentration}.

\subsection{Model of true functions}

To analyze the excess risk, we need to specify a function class (in other words, model) in which the true function $\ftrue$ is included. In this paper, we only consider a two layer neural network model, whereas the techniques adapted in this paper can be directly extended to deeper neural network models.
We consider a teacher-student setting, that is, the true function $\ftrue$ can be represented by a neural network 
defined as follows.
For $w \in \Real$, let $\bar{w}$ be a ``clipping'' of $w$ defined as $\bar{w} := R\times \tanh(w/R)$ 
where $R \geq 1$ is a fixed constant, and let $[x;1] := [x^\top,1]^\top$ for $x \in \Real^d$.
Then, the teacher network is given by 
$$
\textstyle 
f_W(x) = \sum_{m=1}^\infty a_m \wmb{2,m} \sigma_m(\wm{1,m}^\top [x;1]),
$$
where $\wm{1,m} \in \Real^{d+1}$ and $\wm{2,m} \in \Real$ ($m\in \Natural$) are the trainable parameters
(where $W = (\wm{1,m},\wm{2,m})_{m=1}^\infty$), $a_m \in \Real~(m \in \Natural)$ is a fixed scaling parameter, and $\sigma_m:\Real \to \Real$ is an activation function for the $m$-th node.
The reason why we applied the clipping operation to the parameter of the second layer is just for a technical reason to ensure convergence of Langevin dynamics. The dynamics is bounded in high probability in practical situations and the boundedness condition would be removed if further theoretical development of infinite dimensional Langevin dynamics would be achieved.

Let $\calH$ be a set of parameters $W$ such that its squared norm is bounded: $\calH := \{W =(\wm{1,m},\wm{2,m})_{m=1}^\infty  \mid \sum_{m=1}^\infty (\|\wm{1,m}\|^2 + \wm{2,m}^2) < \infty \}$.
Define $\|W\|_{\calH} := 
[\sum_{m=1}^\infty (\|\wm{1,m}\|^2 + \wm{2,m}^2)]^{1/2}$
for $W \in \calH$.
Let $(\mu_m)_{m=1}^\infty$ be a regularization parameter such that $\mu_m \searrow 0$.
Accordingly we define $\calH_{\gamma} := \{W \in \calH \mid  \|W\|_{\calH_\gamma}< \infty \}$
where $\|W\|_{\calH_\gamma} :=
[\sum_{m=1}^\infty \mu_m^{-\gamma}(\|\wm{1,m}\|^2 + \wm{2,m}^2)]^{1/2}
$ for a given $0 < \gamma $.
Throughout this paper, we analyze an estimation problem in which the true function is included in the following model:
$$
\calF_\gamma = \{ f_W \mid W \in \calH_{\gamma}, ~\|W\|_{\calH_\gamma} \leq 1 \}.
$$

This is basically two layer neural network with infinite width.
As assumed later, $a_m$ is assumed to decrease as $m\to \infty$. Its decreasing rate controls the capacity of the model. 
If the first layer parameters $(\wm{1,m})_m$ are fixed, 
this model can be regarded as a variant of the unit ball of some reproducing kernel Hilbert space (RKHS)
with basis functions $a_m \sigma_m(\wm{1,m}^\top [x;1])$. 
However, since the first layer ($\wm{1,m}$) is also trainable, there appears significant difference between deep and kernel approaches.
The Barron class \citep{barron1993universal,weinan2019comparative} is relevant to this function class.
Indeed, it is defined as the convex hull of $w_2 \sigma(w_1^\top [x;1])$ with norm constraints on $(w_1,w_2)$ where $\sigma$ is an activation function.
On the other hand, we will put an explicit decay rate on $a_m$ and the parameter $W$ has an $L_2$-norm constraint, which makes the model $\calF_\gamma$ smaller than the Barron class. 

\section{Estimators}

We consider two classes of estimators and discuss their differences: linear estimators and deep learning estimator with noisy gradient descent (\GLD).

{\bf Linear estimator}~
A class of linear estimators, which we consider as a representative of ``shallow'' learning approach, 
consists of all estimators that have the following form:
$$
\textstyle
\fhat(x) = \sum_{i=1}^n y_i \varphi_i(x_1,\dots,x_n,x).
$$
Here, $(\varphi_i)_{i=1}^n$ can be any measurable function (and $L_2(P_X)$-integrable so that the excess risk can be defined). Thus, they could be selected as the ``optimal'' one so that the corresponding linear estimator minimizes the worst case excess risk.
Even if we chose such an optimal one, the worst case excess risk should be lower bounded by our lower bound given in Theorem \ref{thm:LinearLowerMinimaxRate}. 
It should be noted that the linear estimator does not necessarily imply ``linear model.'' 
The most relevant linear estimator in the machine learning literature is the kernel ridge regression:
$\fhat(x) = Y^\top (K_X + \lambda \Id)^{-1} \mathbf{k}(x)$ where $K_X = (k(x_i,x_j))_{i,j=1}^n \in \Real^{n \times n}$, $\mathbf{k}(x) = [k(x,x_1),\dots,k(x,x_n)]^\top \in \Real^n$ and $Y = [y_1,\dots,y_n]^\top \in \Real^{n}$ for a kernel function $k:\Real^d \times \Real^d \to \Real$. Therefore, the ridge regression estimator in the NTK regime or the random feature model is also included in the class of linear estimators. 
The solution obtained in the early stopping criteria instead of regularization in the NTK regime under the squared loss 
is also included in the linear estimators. Other examples include the $k$-NN estimator and the Nadaraya-Watson estimator.
All of them do not train the basis function in a nonlinear way, which makes difference from the deep learning approach.
In the nonparametric statistics literature, linear estimators have been studied for estimating a wavelet series model. \cite{donoho1990,donoho1996densitywaveletminimax} have shown 
that a wavelet shrinkage estimator can outperform any linear estimator by showing suboptimality of linear estimators. 
\cite{suzuki2018adaptivity} utilized such an argument to show superiority of deep learning but did not present any tractable optimization algorithm.

{\bf Noisy Gradient Descent with regularization}~
As for the neural network approach, we consider a noisy gradient descent algorithm.
Basically, we minimize the following regularized empirical risk:
$$
\textstyle
\calLhat(f_W) + \frac{\lambda}{2} \|W\|^2_{\calH_1}.
$$ 
Here, we employ $\calH_1$-norm as the regularizer. We note that the constant $\gamma$ 
controls the relative complexity of the true function $\ftrue$ compared to the typical solution obtained under the regularization.
Here, we define a linear operator $A$ as $\lambda \|W\|_{\calH_1} = W^\top A W$, that is, 
$A W =   (\lambda \mu_m^{-1} \wm{1,m},\lambda \mu_m^{-1} \wm{2,m})_{m=1}^\infty$.
The regularized empirical risk can be minimized by noisy gradient descent as 
$
W_{k+1} = W_{k} - \eta \nabla (\calLhat(f_{W_{k}}) + \tfrac{\lambda}{2} \|W_{k}\|^2_{\calH_1}) + \sqrt{\frac{2 \eta}{\beta}} \xi_{k},
$
where $\eta > 0$ is a step size and $\xi_{k} = (\xi_{k,(1,m)},\xi_{k,(2,m)})_{m=1}^\infty$ is an infinite-dimensional Gaussian noise, 
i.e., $\xi_{k,(1,m)}$ and $\xi_{k,(2,m)}$ are independently identically distributed from the standard normal distribution
\citep{da_prato_zabczyk_1996}.
Here, $\nabla \calLhat(f_{W}) = \frac{1}{n}\sum_{i=1}^n 2 (f_W(x_i) - y_i) (\wmb{2,m}a_m [x_i;1] \sigma_m'(\wm{1,m}^\top [x_i;1]), 
a_m  \tanh'(\wm{2,m}/R) \sigma_m(\wm{1,m}^\top [x_i;1]))_{m=1}^\infty$.
However, since $\nabla \|W_{k-1}\|^2_{\calH_1}$ is unbounded which makes it difficult to show convergence, 
we employ the {\it semi-implicit Euler scheme} defined by 
\begin{align}
& \textstyle W_{k+1} \!=\! W_{k} \! -\! \eta \nabla \calLhat(f_{W_{k}}) \!-\! \eta A W_{k+1}  \!+\! \sqrt{\frac{2 \eta}{\beta}} \xi_{k} 
\textstyle
\Leftrightarrow 
W_{k+1} \!=\! S_\eta\left(W_{k}\! -\! \eta \nabla \calLhat(f_{W_{k}}) \!+\! \sqrt{\frac{2 \eta}{\beta}} \xi_{k} \right),
\label{eq:SemiEulerAlg}
\end{align}
where $S_\eta := (\Id + \eta A)^{-1}$.
It is easy to check that this is equivalent to the following update rule:
$
\textstyle
W_k = W_{k-1} - \eta \left(\nabla \calLhat(f_{W_{k-1}}) + S_\eta A W_{k-1}  + \sqrt{\frac{2 \eta}{\beta}} \xi_{k-1}\right).
$
Therefore, the implicit Euler scheme can be seen as a naive noisy gradient descent for minimizing the empirical risk with a slightly modified ridge regularization.
This can be interpreted as a discrete time approximation of the following {\it infinite dimensional Langevin dynamics}:
\begin{equation}\label{eq:ContinousLangevin}
\dd W_t = - \nabla (\calLhat(f_{W_t}) + \tfrac{\lambda}{2} \|W_t\|_{\calH_1}^2) \dd t+ \sqrt{2/\beta} \dd \xi_t,
\end{equation}
where $(\xi_t)_{t \geq 0}$ is the so-called cylindrical Brownian motion (see \cite{da_prato_zabczyk_1996} for the details). 
Its application and analysis 
for machine learning problems with non-convex objectives have been recently studied by, for example,  \cite{muzellec2020dimensionfree,Suzuki:NIPS:2020}.

The above mentioned algorithm is executed on an infinite dimensional parameter space. 
In practice, we should deal with a finite width network. To do so, we approximate the solution by a finite dimensional one:
$W^{(M)} = (\wm{1,m},\wm{2,m})_{m=1}^M$ where $M$ corresponds to the width of the network. We identify $W^{(M)}$ to the ``zero-padded'' infinite dimensional one, 
$W = (\wm{1,m},\wm{2,m})_{m=1}^\infty$ with $\wm{1,m} = 0$ and $\wm{2,m}=0$ for all $m > M$. Accordingly, we use the same notation $f_{W^{(M)}}$ to indicate $f_W$ with zero padded vector $W$.
Then, the finite dimensional version of the update rule is given by
$
W_{k+1}^{(M)} = S_\eta^{(M)}\left(W_{k}^{(M)} - \eta \nabla \calLhat(f_{W_{k}^{(M)}}) + \sqrt{\frac{2 \eta}{\beta}} \xi^{(M)}_{k} \right),
$
where $\xi^{(M)}_k$ is the Gaussian noise vector obtained by projecting $\xi_k$ to the first $M$ components 
and $S_\eta^{(M)}$ is also obtained in a similar way.

\section{Convergence rate of estimators}

In this section, we present 
the excess risk bounds for linear estimators and the deep learning estimator. As for the linear estimators, we give its lower bound while we give an upper bound for the deep learning approach.
To obtain the result, we setup some assumptions on the model. 
\begin{Assumption}\label{ass:sigmamam}~
{\vspace{-0.1cm}
\begin{assumenum}[leftmargin=*,topsep=0.5mm,itemsep=0mm] 
\item \label{ass:mumdecay} There exists a constant $c_\mu$ such that $\mu_m \leq c_\mu m^{-2}~(m \in \Natural)$.
\item \label{ass:amdecay} There exists $\aone > 1/2$ such that $a_m \leq  \mu_m^{\aone}~(m \in \Natural)$.
\item \label{ass:activationass} The activation functions $(\sigma_m)_m$ is bounded as $\|\sigma_m\|_\infty \leq 1$.
Moreover, they are three times differentiable and their derivatives upto third order differentiation are uniformly bounded:
$\exists C_\sigma$ such that $\|\sigma_m\|_{1,3} := \max\{\|\sigma_m'\|_\infty,
\|\sigma''_{m}\|_\infty, \|\sigma'''_m\|_\infty\} \leq C_\sigma~(\forall m \in \Natural)$.
\end{assumenum}
\vspace{-0.2cm}
}
\end{Assumption}
The first assumption \ref{ass:mumdecay} controls the strength of the regularization,
and combined with the second assumption \ref{ass:amdecay} and definition of the model $\calF_\gamma$, 
complexity of the model is controlled.
If $\aone$ and $\gamma$ are large, the model is less complicated. Indeed, the convergence rate of the excess risk becomes faster if these parameters are large as seen later.
The decay rate $\mu_m \leq c_\mu m^{-2}$ can be generalized as $m^{-p}$ with $p > 1$ but we employ this setting just for a technical simplicity for ensuring convergence of the Langevin dynamics.
The third assumption \ref{ass:activationass} can be satisfied by several activation functions such as the sigmoid function and the hyperbolic tangent.
The assumption $\|\sigma_m\|_\infty \leq 1$ could be replaced by another one like $\|\sigma_m\|_\infty \leq C$, but we fix this scaling for simple presentation. 

\subsection{Minimax lower bound for linear estimators}

Here, we analyze a lower bound of excess risk of linear estimators, and eventually we show that {\it any} linear estimator suffers from curse of dimensionality.  
To rigorously show that, we consider the following minimax excess risk over the class of linear estimators: 
$$
\Rlin(\calF_\gamma) :=\inf_{\fhat: \text{linear} } \sup_{\ftrue \in \calF_\gamma} \EE_{D_n}[\|\fhat - \ftrue\|_{\LPiPx}^2],
$$
where $\inf$ is taken over all linear estimators and $\EE_{D_n}[\cdot]$ is taken with respect to the training data $D_n$. 
This expresses the best achievable worst case error over the class of linear estimators to estimate a function in $\calF_\gamma$.
To evaluate it, 
we additionally assume the following condition.
\begin{Assumption}\label{ass:MinimaxAss}
We assume that $\mu_m = m^{-2}$ and $a_m =  \mu_m^{\aone}~(m \in \Natural)$ (and hence $c_\mu = 1$).
There exists a monotonically decreasing sequence $(b_m)_{m=1}^\infty$ and $s \geq 3$ such that 
$b_m =\mu_m^{\atwo} ~(\forall m)$ with $\atwo > \gamma/2$ and 
$\sigma_m(u) = b_m^s \sigma(b_m^{-1} u)~(u \in \Real)$ where $\sigma$ is the sigmoid function: $\sigma(u) = 1/(1+e^{-u})$.
\end{Assumption}
Intuitively, the parameter $s$ 
controls the ``resolution'' of each basis function $\sigma_m$,
and the relation between parameter $\aone$ and $\atwo$ controls the magnitude of coefficient for each basis $\sigma_m$.
Note that the condition $s \geq 3$ ensures $\|\sigma_m\|_{1,3}$ is uniformly bounded and $0 < b_m \leq 1$ ensures $\|\sigma_m\|_\infty \leq 1$.
Our main strategy to obtain the lower bound is to make use of the so-called {\it convex-hull argument}.
That is, it is known that, for a function class $\calF$, the minimax risk $R(\calF)$ over a class of linear estimators is identical to that for the convex hull of $\calF$ \citep{HAYAKAWA2020343,donoho1990}:
$$
\Rlin(\calF) = \Rlin(\convbar(\calF)),
$$
where $\conv(\calF) = \{\sum_{i=1}^N \lambda_i f_i \mid f_i \in \calF,~\sum_{i=1}^N \lambda_i = 1,~\lambda_i \geq 0,~N \in \Natural \}$ and $\convbar(\cdot)$ is the closure of $\conv(\cdot)$ with respect to $L_2(\Px)$-norm.
Intuitively, since the linear estimator is linear to the observations $(y_i)_{i=1}^n$ of outputs, 
a simple application of Jensen's inequality yields that its worst case error on the convex hull of the function class $\calF$ does not increase compared with that on the original one $\calF$ (see \cite{HAYAKAWA2020343} for the details). 
This indicates that the linear estimators cannot distinguish the original hypothesis class $\calF$ and its convex hull.
Therefore, if the class $\calF$ is highly non-convex, then the linear estimators suffer from much slower convergence rate because its convex hull $\convbar(\calF)$ becomes much ``fatter'' than the original one $\calF$.
To make use of this argument, for each sample size $n$, we pick up appropriate $m_n$ and consider a subset generated by the basis function $\sigma_{m_n}$, i.e., $\calF_\gamma^{(n)} := \{ a_{m_n} \wmb{2,m_n} \sigma_m(\wm{1,m_n}^\top [x;1]) \in \calF_\gamma \}$. 
By applying the convex hull argument to this set, we obtain the relation $\Rlin(\calF_\gamma) \geq \Rlin(\calF_{\gamma}^{(n)}) = \Rlin(\convbar(\calF_{\gamma}^{(n)}))$.
Since $\calF_\gamma^{(n)}$ is highly non-convex, its convex hull $\convbar(\calF_{\gamma}^{(n)})$ is much larger than the original set $\calF_{\gamma}^{(n)}$ and thus the minimax risk over the linear estimators would be much larger than that over all estimators including deep learning.
More intuitively, linear estimators do not adaptively select the basis functions and thus they should prepare redundantly large class of basis functions to approximate functions in the target function class. 
The following theorem gives the lower bound of the minimax optimal excess risk over the class of linear estimators.

\begin{Theorem}\label{thm:LinearLowerMinimaxRate}
Suppose that $\mathrm{Var}(\epsilon) > 0$, $\Px$ is the uniform distribution on $[0,1]^d$, and Assumption \ref{ass:MinimaxAss} is satisfied.
Let $\tilde{\beta} = \frac{\aone + (s+1) \atwo}{\atwo - \gamma/2}$.
Then for arbitrary small $\kappa' > 0$, we have that 
\begin{equation}\label{eq:LinLowerRate}
\Rlin(\calF_\gamma) \gtrsim n^{- \frac{2\tilde{\beta} + d }{2\tilde{\beta} + 2d}} n^{-\kappa'}.
\end{equation}
\end{Theorem}
The proof is in Appendix \ref{sec:ProofMinimaxLowerLinear}. 
We utilized the Irie-Miyake integral representation \citep{ICNN:IrieMiyake:1988,NN:HORNIK:1990551} to show there exists a ``complicated'' function in the convex hull, and then we adopted the technique of \cite{zhang2002wavelet} to show the lower bound.
The lower bound is characterized by the decaying rate ($\aone$) of $a_m$ relative to that ($\atwo$) of the scaling factor $b_m$.
Indeed, the faster $a_m$ decays with increasing $m$, the faster the rate of the minimax lower bound becomes.
We can see that the minimax rate of linear estimators is quite sensitive to the dimension $d$.
Actually, for relatively high dimensional settings, this lower bound becomes close to a slow rate $\Omega(1/\sqrt{n})$, which corresponds to the curse of dimensionality.

It has been pointed out that the sample complexity of kernel methods suffers from the curse of dimensionality while deep learning can avoid that by a tractable algorithms (e.g., \cite{ghorbani2019linearized,JMLR:v18:14-546}).
Among them, \cite{ghorbani2019linearized} showed that if the dimensionality $d$ is polynomial against $n$, then the excess risk of kernel methods is bounded away from 0 for all $n$. 
On the other hand, our analysis can be applied to {\it any} linear estimator including kernel methods,
and it shows that even if the dimensionality $d$ is fixed, the convergence rate of their excess risk suffers from the curse of dimensionality. This can be accomplished thanks to a careful analysis of the rate of convergence.
\cite{JMLR:v18:14-546} derived an upper bound of the Rademacher complexity of the unit ball of the RKHS corresponding to a neural network model. However, it is just an upper bound and there is still a large gap from excess risk estimates. 
\cite{AllenZhu:NIPS2019:BeyondKernel,allen2020backward,Bai2020Beyond,chen2020understanding} also analyzed a lower bound of sample complexity of kernel methods. However, their lower bound is not for the excess risk of the squared loss. Eventually, the sample complexities of all methods including deep learning take a form of $O(C/\sqrt{n})$ and dependency of coefficient $C$ to the dimensionality or other factors such as magnitude of residual components is compared. On the other hand, our lower bound properly involves the properties of squared loss such as strong convexity and smoothness 
and the bound shows the curse of dimensionality occurs even in the rate of convergence instead of just the coefficient. 
Finally, we would like to point out that several existing work (e.g., \cite{ghorbani2019linearized,AllenZhu:NIPS2019:BeyondKernel}) considered a situation where the target function class changes as the sample size $n$ increases. However, our analysis reveals that separation between deep and shallow occurs even if the target function class $\calF_\gamma$ is fixed.

\subsection{Upper bound for deep learning}

Here, we analyze the excess risk of deep learning trained by \GLD and its algorithmic convergence rate.
Our analysis heavily relies on the weak convergence of the discrete time gradient Langevin dynamics to 
the stationary distribution of the continuous time one (\Eqref{eq:ContinousLangevin}).
Under some assumptions, the continuous time dynamics has a stationary distribution \citep{DaPrato_Zabczyk92,Maslowski:1989,Sowers:1992,Jacquot+Gilles:1995,Shardlow:1999,Hairer:2002}.
If we denote the probability measure on $\calH$ corresponding to the stationary distribution by $\pi_\infty$, 
then it is given by 
\begin{align*}
\textstyle 
\frac{\dd \pi_\infty}{\dd \nu_\beta}(W) \propto \exp(- \beta \calLhat(f_W)), 
\end{align*}
where $\nu_\beta$ is the Gaussian measure in $\cH$ with mean 0 and covariance $(\beta A)^{-1}$ (see \citet{da_prato_zabczyk_1996} for the rigorous definition of the Gaussian measure on a Hilbert space).
Remarkably, this can be seen as {\it the Bayes posterior} for a prior distribution $\nu_\beta$ and a ``log-likelihood'' function $\exp(- \beta \calLhat(W))$. Through this view point, we can obtain an excess risk bound of the solution $W_k$.
The proofs of all theorems in this section are in Appendix \ref{sec:ProofConvDeep}.

Under Assumption \ref{ass:sigmamam}, the distribution of $W_k$ derived by the discrete time gradient Langevin synamics satisfies the following weak convergence property to the stationary distribution $\pi_\infty$.
This convergence rate analysis depends on the techniques by \cite{Brehier16,muzellec2020dimensionfree}. 

\begin{Proposition}\label{prop:WeakConvergence}
Assume Assumption \ref{ass:sigmamam} holds and $\beta > \eta$. 
Then, there exist spectral gaps $\Lambda^*_\eta$ and $\Lambda^*_0$ (defined in \Eqref{eq:SpectralGap} of Appendix \ref{sec:AuxLemmas})
and a constant $C_{0}$
such that, for any $0 < a < 1/4$, the following convergence bound holds for almost sure observation $D_n$:  
\begin{align} 
&| \EE_{W_k}[\calL(f_{W_k}) | D_n] -\EE_{W\sim \pi_\infty}[\calL(f_W)|D_n] | 
\leq
C_{0}
\exp(- \Lambda_\eta^* \eta k )   + 
C_1 \frac{\sqrt{\beta}}{\Lambda^*_0}\eta^{1/2-a} 
=: \Xi_k,
\end{align}
where $C_1$ is a constant depending only on $c_\mu,R,\aone,C_\sigma,U,a$ (independent of $\eta,k,\beta,\lambda,n$).
\end{Proposition}
This proposition indicates that the expected risk of $W_k$ can be almost identical to that of the ``Bayes posterior solution'' obeying $\pi_\infty$ after sufficiently large iterations $k$ with sufficiently small step size $\eta$ even though $\calLhat(f_W)$ is not convex. 
The definition of $\Lambda^*_\eta$ can be found in \Eqref{eq:SpectralGap}. 
We should note that its dependency on $\beta$ is exponential. Thus, if we take $\beta = \Omega(n)$, 
then the computational cost until a sufficiently small error could be exponential with respect to the sample size $n$.
The same convergence holds also for finite dimensional one $W_k^{(M)}$ with a modified stationary distribution.
The constants appearing in the bound are independent of the model size $M$
(see the proof of Proposition \ref{prop:WeakConvergence} in Appendix \ref{sec:ProofConvDeep}). 
In particular, the convergence can be guaranteed even if $W$ is infinite dimensional. This is quite different from usual finite dimensional analyses \citep{Raginsky_Rakhlin_Telgarsky2017,NIPS2018_8175,Xu_Chen_Zou_Gu18} which requires exponential dependency on the dimension, but thanks to the regularization term, we can obtain the model size independent convergence rate.
\cite{Xu_Chen_Zou_Gu18} also analyzed a finite dimensional gradient Langevin dynamics and obtained a similar bound where $O(\eta)$ appears in place of the second term $\eta^{1/2-a}$ which corresponds to time discretization error.
In our setting the regularization term is $\|W\|^2_{\calH_1} = \sum_m (\|w_{1,m}\|^2 + w_{2,m}^2)/\mu_m$ with $\mu_m \lesssim m^{-2}$, but if we employ $\|W\|^2_{\calH_{p/2}} = \sum_m (\|w_{1,m}\|^2 + w_{2,m}^2)/\mu_m^{p/2}$ for $p > 1$, then the time discretization error term would be modified to $\eta^{(p-1)/p-a}$ \citep{Andersson2016}.
We can interpret the finite dimensional setting as the limit of $p \to \infty$ which leads to $\eta^{(p-1)/p} \to \eta$ that recovers the finite dimensional result ($O(\eta)$) as shown by \cite{Xu_Chen_Zou_Gu18}.

In addition to the above algorithmic convergence, we also have the following convergence rate for the excess risk bound of the finite dimensional solution $W_k^{(M)}$. 
\begin{Theorem}\label{thm:ExcessRiskConvRate}
Assume 
Assumption \ref{ass:sigmamam} holds, 
assume $\eta < \beta \leq \min\{ n/(2U^2),n\}$, and $0 < \gamma < 1/2+\aone$. 
Then, 
if the width satisfies $M \geq  \min\left\{\lambda^{1/4\gamma(\aone + 1)} \beta^{1/2\gamma}, \lambda^{-1/2(\aone + 1)},n^{1/2\gamma}\right\}$, 
the expected excess risk of $W_k$ is bounded as 
\begin{align*}
\EE_{D^n}\!\!\left[ \EE_{W_k^{(\!M\!)}}[ \|f_{W_k^{(\!M\!)}} \!\!-\! \ftrue\|_{\LPiPx}^2 |D_n] \right]
\! 
\leq  \!
 C 
\max \!\big\{ \! (\lambda\beta)^{\frac{1/\gamma}{1 + 1/2\gamma}} n^{-\! \frac{1}{1 + 1/2\gamma}}\!\!,
\lambda^{-\frac{1}{2(\aone + 1)}} \beta^{-1}\!\!,  \lambda^{\frac{\gamma}{1 + \aone}}\big\} \! + \! \Xi_k,
%
\end{align*}
where $C$ is a constant independent of $n,\beta,\lambda,\eta,k$.
In particular, if we set $\beta = \min\{n/(2U^2),n\}$ and $\lambda = \beta^{-1}$, 
then for $M \geq n^{1/2(\aone + 1)}$, we obtain 
\begin{equation*}
\EE_{D^n}\left[ \EE_{W_k^{(M)}}[ \|f_{W_k^{(M)}} - \ftrue\|_{\LPiPx}^2 |D_n] \right]
\lesssim 
n^{-\frac{\gamma}{\aone + 1}} + \Xi_k.
\end{equation*}
\end{Theorem}

In addition to this theorem, if we further assume Assumption \ref{ass:MinimaxAss}, we obtain a refined bound as follows.
\begin{Corollary}\label{cor:RefinedExcessBound}
Assume Assumptions \ref{ass:sigmamam} and \ref{ass:MinimaxAss} hold and $\eta < \beta$, and 
let $\beta = \min\{ n/(2U^2),n\}$ and $\lambda=\beta^{-1}$.
Suppose that there exists $0 \leq q \leq s-3$ such that $0 < \gamma < 1/2+\aone + q \atwo$. 
Then, the excess risk bound of $W_k^{(M)}$ for $M \geq n^{1/2(\aone + q \atwo + 1)}$ can be refined as 
\begin{equation}\label{eq:DLConvRate}
\EE_{D^n}\left[ \EE_{W_k^{(M)}}[ \|f_{W_k^{(M)}} - \ftrue\|_{\LPiPx}^2 |D_n] \right]
\lesssim n^{-\frac{\gamma}{\aone + q \atwo + 1}}+ \Xi_k.
\end{equation}

\end{Corollary}

These theorem and corollary shows that the tractable \GLD algorithm achieves a fast convergence rate of the excess risk bound. Indeed, if $q$ is chosen so that $\gamma > (\aone + q \atwo + 1)/2$, then the excess risk bound converges faster than $O(1/\sqrt{n})$. Remarkably, the convergence rate is not affected by the input dimension $d$, which makes discrepancy from linear estimators.
The bound of Theorem \ref{thm:ExcessRiskConvRate} is tightest when $\gamma$ is close to $1/2 + \aone$ ($\gamma \approx 1/2+\aone + 3 \atwo$ for Corollary \ref{cor:RefinedExcessBound}), and a smaller $\gamma$ yields looser bound.
This relation between $\gamma$ and $\aone$ reflects misspecification of the ``prior'' distribution. 
When $\gamma$ is small, the regularization $\lambda \|W\|_{\calH_1}^2$ is not strong enough so that the variance of the posterior distribution becomes unnecessarily large for estimating the true function $\ftrue \in \calF_\gamma$.  
Therefore, the best achievable bound can be obtained when the regularization is correctly specified.
The analysis of fast rate is in contrast to some existing work \citep{AllenZhu:NIPS2019:BeyondKernel,allen2020backward,pmlr-v125-li20a,Bai2020Beyond} that basically evaluated the Rademacher complexity. 
This is because we essentially evaluated a local Rademacher complexity instead.

\subsection{Comparison between linear estimators and deep learning}

Here, we compare the convergence rate of excess risks between the linear estimators and the neural network method trained by \GLD using the bounds obtained in Theorem \ref{thm:LinearLowerMinimaxRate} and Corollary \ref{cor:RefinedExcessBound} respectively.
We write the lower bound \eqref{eq:LinLowerRate} of the minimax excess risk of linear estimators as $\Rlin^*$ and 
the excess risk of the neural network approach \eqref{eq:DLConvRate} as $\Rnn^*$.
To make the discussion concise, we consider a specific situation where $s =3$, $\aone = \gamma = \tfrac{1}{4}\atwo$. In this case, $\tilde{\beta} =17/3 \approx 5.667$, which gives  
$$
\textstyle
\Rlin^* \gtrsim n^{-\left(1 + \tfrac{d}{2 \tilde{\beta} + d}\right)^{-1}}n^{-\kappa'}
\approx n^{-\left(1 + \tfrac{d}{11.3 + d}\right)^{-1}}n^{-\kappa'}. 
$$
On the other hand, by setting $q=0$, we have
$$
\textstyle
\Rnn^* \lesssim n^{-\frac{\aone}{\aone + 1}} = n^{-\left(1 + \tfrac{1}{\aone}\right)^{-1}}.
$$
Thus, as long as $\aone > 11.3/d + 1 \approx 2 \tilde{\beta}/d+1$, we have that 
$$
\textstyle
\Rlin^* \gtrsim \Rnn^*,~~\text{and}~~\lim_{n \to \infty} \frac{\Rnn^*}{\Rlin^*} = 0.
$$ 
In particular, as $d$ gets larger, $\Rlin^*$ approaches to $\Omega(n^{-1/2})$ while $\Rnn^*$ is not affected by $d$ and it gets close to $O(n^{-1})$ as $\aone$ gets larger.
Moreover, the inequality $\aone >  11.3/d + 1$ can be satisfied by a relatively low dimensional setting;
for example, $d = 10$ is sufficient when $\aone = 3$.
As $\aone$ becomes large, the model becomes ``simpler'' because $(a_m)_{m=1}^\infty$ decays faster. 
However, the linear estimators cannot take advantage of this information whereas deep learning can. 
From the convex hull argument, this discrepancy stems from the non-convexity of the model.
We also note that the superiority of deep learning is shown {\it without} sparse regularization while several existing work showed favorable estimation property of deep learning though sparsity inducing regularization \citep{JMLR:v18:14-546,chizat2019sparse,HAYAKAWA2020343}.
However, our analysis indicates that sparse regularization is not necessarily as long as the model has non-convex geometry, i.e., 
sparsity is just one sufficient condition for non-convexity but not a necessarily condition.
The parameter setting above is just a sufficient condition and the lower bound $\Rlin^*$ would not be tight.
The superiority of deep learning would hold in much wider situations.


\section{Conclusion}
In this paper, we studied excess risks of linear estimators, as a representative of shallow methods, and a neural network estimator trained by a noisy gradient descent
where the model is fixed and no sparsity inducing regularization is imposed.
Our analysis revealed that deep learning can outperform any linear estimator even for a relatively low dimensional setting.
Essentially, non-convexity of the model induces this difference and the curse of dimensionality for linear estimators is a consequence of a fact that the geometry of the model becomes more ``non-convex'' as the dimension of input gets higher.
All derived bounds are fast rate because the analyses are about the excess risk with the squared loss,
which made it possible to compare the rate of convergence.
The fast learning rate of the deep learning approach is derived through the fact that the noisy gradient descent behaves like a Bayes estimator with model size independent convergence rate.


\subsubsection*{Acknowledgments}
TS was partially supported by JSPS Kakenhi
(26280009, 15H05707 and 18H03201), Japan Digital Design and JST-CREST.


\bibliography{main} 

\begin{thebibliography}{62}
\providecommand{\natexlab}[1]{#1}
\providecommand{\url}[1]{\texttt{#1}}
\expandafter\ifx\csname urlstyle\endcsname\relax
  \providecommand{\doi}[1]{doi: #1}\else
  \providecommand{\doi}{doi: \begingroup \urlstyle{rm}\Url}\fi

\bibitem[Allen-Zhu \& Li(2019)Allen-Zhu and Li]{AllenZhu:NIPS2019:BeyondKernel}
Z.~Allen-Zhu and Y.~Li.
\newblock What can {ResNet} learn efficiently, going beyond kernels?
\newblock In \emph{Advances in Neural Information Processing Systems 32}, pp.\
  9017--9028. Curran Associates, Inc., 2019.

\bibitem[Allen-Zhu \& Li(2020)Allen-Zhu and Li]{allen2020backward}
Z.~Allen-Zhu and Y.~Li.
\newblock Backward feature correction: How deep learning performs deep
  learning.
\newblock \emph{arXiv preprint arXiv:2001.04413}, 2020.

\bibitem[Allen-Zhu et~al.(2019)Allen-Zhu, Li, and Song]{allen2019convergence}
Z.~Allen-Zhu, Y.~Li, and Z.~Song.
\newblock A convergence theory for deep learning via over-parameterization.
\newblock In \emph{Proceedings of International Conference on Machine
  Learning}, pp.\  242--252, 2019.

\bibitem[Andersson et~al.(2016)Andersson, Kruse, and Larsson]{Andersson2016}
A.~Andersson, R.~Kruse, and S.~Larsson.
\newblock Duality in refined {Sobolev--Malliavin} spaces and weak approximation
  of {SPDE}.
\newblock \emph{Stochastics and Partial Differential Equations Analysis and
  Computations}, 4\penalty0 (1):\penalty0 113--149, 2016.

\bibitem[Arora et~al.(2019)Arora, Du, Hu, Li, and Wang]{arora2019fine}
S.~Arora, S.~S. Du, W.~Hu, Z.~Li, and R.~Wang.
\newblock Fine-grained analysis of optimization and generalization for
  overparameterized two-layer neural networks.
\newblock \emph{arXiv preprint arXiv:1901.08584}, 2019.

\bibitem[Bach(2017)]{JMLR:v18:14-546}
F.~Bach.
\newblock Breaking the curse of dimensionality with convex neural networks.
\newblock \emph{Journal of Machine Learning Research}, 18\penalty0
  (19):\penalty0 1--53, 2017.

\bibitem[Bai \& Lee(2020)Bai and Lee]{Bai2020Beyond}
Y.~Bai and J.~D. Lee.
\newblock Beyond linearization: On quadratic and higher-order approximation of
  wide neural networks.
\newblock In \emph{International Conference on Learning Representations}, 2020.
\newblock URL \url{https://openreview.net/forum?id=rkllGyBFPH}.

\bibitem[Barron(1993)]{barron1993universal}
A.~R. Barron.
\newblock Universal approximation bounds for superpositions of a sigmoidal
  function.
\newblock \emph{IEEE Transactions on Information theory}, 39\penalty0
  (3):\penalty0 930--945, 1993.

\bibitem[Bartlett et~al.(2005)Bartlett, Bousquet, and
  Mendelson]{LocalRademacher}
P.~Bartlett, O.~Bousquet, and S.~Mendelson.
\newblock Local {R}ademacher complexities.
\newblock \emph{The Annals of Statistics}, 33:\penalty0 1487--1537, 2005.

\bibitem[Br{\'e}hier \& Kopec(2016)Br{\'e}hier and Kopec]{Brehier16}
C.-E. Br{\'e}hier and M.~Kopec.
\newblock Approximation of the invariant law of {SPDEs}: error analysis using a
  {Poisson} equation for a full-discretization scheme.
\newblock \emph{IMA Journal of Numerical Analysis}, 37\penalty0 (3):\penalty0
  1375--1410, 07 2016.

\bibitem[Cao \& Gu(2019)Cao and Gu]{cao2019generalization}
Y.~Cao and Q.~Gu.
\newblock A generalization theory of gradient descent for learning
  over-parameterized deep {ReLU} networks.
\newblock \emph{arXiv preprint arXiv:1902.01384}, 2019.

\bibitem[Caponnetto \& {de Vito}(2007)Caponnetto and {de
  Vito}]{FCM:Caponetto+Vito:2007}
A.~Caponnetto and E.~{de Vito}.
\newblock Optimal rates for regularized least-squares algorithm.
\newblock \emph{Foundations of Computational Mathematics}, 7\penalty0
  (3):\penalty0 331--368, 2007.

\bibitem[Chen et~al.(2020)Chen, Bai, Lee, Zhao, Wang, Xiong, and
  Socher]{chen2020understanding}
M.~Chen, Y.~Bai, J.~D. Lee, T.~Zhao, H.~Wang, C.~Xiong, and R.~Socher.
\newblock Towards understanding hierarchical learning: Benefits of neural
  representations.
\newblock \emph{Advances in Neural Information Processing Systems}, 33, 2020.

\bibitem[Chizat(2019)]{chizat2019sparse}
L.~Chizat.
\newblock Sparse optimization on measures with over-parameterized gradient
  descent.
\newblock \emph{arXiv preprint arXiv:1907.10300}, 2019.

\bibitem[Chizat \& Bach(2018)Chizat and Bach]{chizat2018note}
L.~Chizat and F.~Bach.
\newblock A note on lazy training in supervised differentiable programming.
\newblock \emph{arXiv preprint arXiv:1812.07956}, 2018.

\bibitem[Chizat \& Bach(2020)Chizat and Bach]{chizat2020implicit}
L.~Chizat and F.~Bach.
\newblock Implicit bias of gradient descent for wide two-layer neural networks
  trained with the logistic loss.
\newblock \emph{arXiv preprint arXiv:2002.04486}, 2020.

\bibitem[Da~Prato \& Zabczyk(1992)Da~Prato and Zabczyk]{DaPrato_Zabczyk92}
G.~Da~Prato and J.~Zabczyk.
\newblock Non-explosion, boundedness and ergodicity for stochastic semilinear
  equations.
\newblock \emph{Journal of Differential Equations}, 98:\penalty0 181--195,
  1992.

\bibitem[Da~Prato \& Zabczyk(1996)Da~Prato and Zabczyk]{da_prato_zabczyk_1996}
G.~Da~Prato and J.~Zabczyk.
\newblock \emph{Ergodicity for Infinite Dimensional Systems}.
\newblock London Mathematical Society Lecture Note Series. Cambridge University
  Press, 1996.

\bibitem[Donoho et~al.(1990)Donoho, Liu, and MacGibbon]{donoho1990}
D.~L. Donoho, R.~C. Liu, and B.~MacGibbon.
\newblock Minimax risk over hyperrectangles, and implications.
\newblock \emph{The Annal of Statistics}, 18\penalty0 (3):\penalty0 1416--1437,
  09 1990.
\newblock \doi{10.1214/aos/1176347758}.

\bibitem[Donoho et~al.(1996)Donoho, Johnstone, Kerkyacharian, and
  Picard]{donoho1996densitywaveletminimax}
D.~L. Donoho, I.~M. Johnstone, G.~Kerkyacharian, and D.~Picard.
\newblock Density estimation by wavelet thresholding.
\newblock \emph{The Annals of Statistics}, 24\penalty0 (2):\penalty0 508--539,
  1996.

\bibitem[Du et~al.(2019{\natexlab{a}})Du, Lee, Li, Wang, and
  Zhai]{du2019gradient}
S.~Du, J.~Lee, H.~Li, L.~Wang, and X.~Zhai.
\newblock Gradient descent finds global minima of deep neural networks.
\newblock In \emph{International Conference on Machine Learning}, pp.\
  1675--1685, 2019{\natexlab{a}}.

\bibitem[Du et~al.(2019{\natexlab{b}})Du, Zhai, Poczos, and
  Singh]{du2018gradient}
S.~S. Du, X.~Zhai, B.~Poczos, and A.~Singh.
\newblock Gradient descent provably optimizes over-parameterized neural
  networks.
\newblock \emph{International Conference on Learning Representations 7},
  2019{\natexlab{b}}.

\bibitem[E et~al.(2019{\natexlab{a}})E, Ma, and Wu]{ma2019priori}
W.~E, C.~Ma, and L.~Wu.
\newblock A priori estimates of the population risk for two-layer neural
  networks.
\newblock \emph{Communications in Mathematical Sciences}, 17\penalty0
  (5):\penalty0 1407--1425, 2019{\natexlab{a}}.

\bibitem[E et~al.(2019{\natexlab{b}})E, Ma, and Wu]{weinan2019comparative}
W.~E, C.~Ma, and L.~Wu.
\newblock A comparative analysis of optimization and generalization properties
  of two-layer neural network and random feature models under gradient descent
  dynamics.
\newblock \emph{Science China Mathematics}, pp.\  1--24, 2019{\natexlab{b}}.

\bibitem[Erdogdu et~al.(2018)Erdogdu, Mackey, and Shamir]{NIPS2018_8175}
M.~A. Erdogdu, L.~Mackey, and O.~Shamir.
\newblock Global non-convex optimization with discretized diffusions.
\newblock In \emph{Advances in Neural Information Processing Systems 31}, pp.\
  9671--9680. 2018.

\bibitem[Ghorbani et~al.(2019)Ghorbani, Mei, Misiakiewicz, and
  Montanari]{ghorbani2019linearized}
B.~Ghorbani, S.~Mei, T.~Misiakiewicz, and A.~Montanari.
\newblock Linearized two-layers neural networks in high dimension.
\newblock \emph{arXiv preprint arXiv:1904.12191}, 2019.

\bibitem[Ghorbani et~al.(2020)Ghorbani, Mei, Misiakiewicz, and
  Montanari]{ghorbani2020neural}
B.~Ghorbani, S.~Mei, T.~Misiakiewicz, and A.~Montanari.
\newblock When do neural networks outperform kernel methods?
\newblock \emph{arXiv preprint arXiv:2006.13409}, 2020.

\bibitem[Gin{\'e} \& Koltchinskii(2006)Gin{\'e} and
  Koltchinskii]{gine2006concentration}
E.~Gin{\'e} and V.~Koltchinskii.
\newblock Concentration inequalities and asymptotic results for ratio type
  empirical processes.
\newblock \emph{The Annals of Probability}, 34\penalty0 (3):\penalty0
  1143--1216, 2006.

\bibitem[Gunasekar et~al.(2018)Gunasekar, Lee, Soudry, and
  Srebro]{gunasekar2018implicit}
S.~Gunasekar, J.~D. Lee, D.~Soudry, and N.~Srebro.
\newblock Implicit bias of gradient descent on linear convolutional networks.
\newblock In \emph{Advances in Neural Information Processing Systems}, pp.\
  9482--9491, 2018.

\bibitem[Hairer(2002)]{Hairer:2002}
M.~Hairer.
\newblock Exponential mixing properties of stochastic {PDE}s through asymptotic
  coupling.
\newblock \emph{Probab. Theory Related Fields}, 124\penalty0 (3):\penalty0
  345--380, 2002.

\bibitem[Hayakawa \& Suzuki(2020)Hayakawa and Suzuki]{HAYAKAWA2020343}
S.~Hayakawa and T.~Suzuki.
\newblock On the minimax optimality and superiority of deep neural network
  learning over sparse parameter spaces.
\newblock \emph{Neural Networks}, 123:\penalty0 343--361, 2020.
\newblock ISSN 0893-6080.

\bibitem[Hornik et~al.(1990)Hornik, Stinchcombe, and White]{NN:HORNIK:1990551}
K.~Hornik, M.~Stinchcombe, and H.~White.
\newblock Universal approximation of an unknown mapping and its derivatives
  using multilayer feedforward networks.
\newblock \emph{Neural Networks}, 3\penalty0 (5):\penalty0 551--560, 1990.

\bibitem[Imaizumi \& Fukumizu(2019)Imaizumi and Fukumizu]{pmlr-v89-imaizumi19a}
M.~Imaizumi and K.~Fukumizu.
\newblock Deep neural networks learn non-smooth functions effectively.
\newblock In K.~Chaudhuri and M.~Sugiyama (eds.), \emph{Proceedings of Machine
  Learning Research}, volume~89 of \emph{Proceedings of Machine Learning
  Research}, pp.\  869--878. PMLR, 16--18 Apr 2019.

\bibitem[Irie \& Miyake(1988)Irie and Miyake]{ICNN:IrieMiyake:1988}
B.~Irie and S.~Miyake.
\newblock Capabilities of three-layered perceptrons.
\newblock In \emph{IEEE 1988 International Conference on Neural Networks}, pp.\
   641--648, 1988.

\bibitem[Jacot et~al.(2018)Jacot, Gabriel, and Hongler]{jacot2018neural}
A.~Jacot, F.~Gabriel, and C.~Hongler.
\newblock Neural tangent kernel: Convergence and generalization in neural
  networks.
\newblock In \emph{Advances in Neural Information Processing Systems 31}, pp.\
  8580--8589, 2018.

\bibitem[Jacquot \& Royer(1995)Jacquot and Royer]{Jacquot+Gilles:1995}
S.~Jacquot and G.~Royer.
\newblock Ergodicit\'{e} d'une classe d'\'{e}quations aux d\'{e}riv\'{e}es
  partielles stochastiques.
\newblock \emph{Comptes Rendus de l'Acad\'{e}mie des Sciences. S\'{e}rie I.
  Math\'{e}matique}, 320\penalty0 (2):\penalty0 231--236, 1995.

\bibitem[Klusowski \& Barron(2016)Klusowski and Barron]{klusowski2016risk}
J.~M. Klusowski and A.~R. Barron.
\newblock Risk bounds for high-dimensional ridge function combinations
  including neural networks.
\newblock \emph{arXiv preprint arXiv:1607.01434}, 2016.

\bibitem[Koltchinskii(2006)]{Koltchinskii}
V.~Koltchinskii.
\newblock Local {R}ademacher complexities and oracle inequalities in risk
  minimization.
\newblock \emph{The Annals of Statistics}, 34:\penalty0 2593--2656, 2006.

\bibitem[Li et~al.(2020)Li, Ma, and Zhang]{pmlr-v125-li20a}
Y.~Li, T.~Ma, and H.~R. Zhang.
\newblock Learning over-parametrized two-layer neural networks beyond ntk.
\newblock volume 125 of \emph{Proceedings of Machine Learning Research}, pp.\
  2613--2682. PMLR, 09--12 Jul 2020.

\bibitem[Maslowski(1989)]{Maslowski:1989}
B.~Maslowski.
\newblock Strong {F}eller property for semilinear stochastic evolution
  equations and applications.
\newblock In \emph{Stochastic systems and optimization ({W}arsaw, 1988)},
  volume 136 of \emph{Lect. Notes Control Inf. Sci.}, pp.\  210--224. Springer,
  Berlin, 1989.

\bibitem[Mei et~al.(2018)Mei, Montanari, and Nguyen]{MeiE7665}
S.~Mei, A.~Montanari, and P.-M. Nguyen.
\newblock A mean field view of the landscape of two-layer neural networks.
\newblock \emph{Proceedings of the National Academy of Sciences}, 115\penalty0
  (33):\penalty0 E7665--E7671, 2018.
\newblock \doi{10.1073/pnas.1806579115}.

\bibitem[Mei et~al.(2019)Mei, Misiakiewicz, and Montanari]{pmlr-v99-mei19a}
S.~Mei, T.~Misiakiewicz, and A.~Montanari.
\newblock Mean-field theory of two-layers neural networks: dimension-free
  bounds and kernel limit.
\newblock In A.~Beygelzimer and D.~Hsu (eds.), \emph{Proceedings of the
  Thirty-Second Conference on Learning Theory}, volume~99 of \emph{Proceedings
  of Machine Learning Research}, pp.\  2388--2464, Phoenix, USA, 25--28 Jun
  2019. PMLR.

\bibitem[Mendelson(2002)]{IEEEIT:Mendelson:2002}
S.~Mendelson.
\newblock Improving the sample complexity using global data.
\newblock \emph{IEEE Transactions on Information Theory}, 48:\penalty0
  1977--1991, 2002.

\bibitem[Mohri et~al.(2012)Mohri, Rostamizadeh, and
  Talwalkar]{mohri2012foundations}
M.~Mohri, A.~Rostamizadeh, and A.~Talwalkar.
\newblock \emph{Foundations of Machine Learning}.
\newblock The MIT Press, 2012.

\bibitem[Muzellec et~al.(2020)Muzellec, Sato, Massias, and
  Suzuki]{muzellec2020dimensionfree}
B.~Muzellec, K.~Sato, M.~Massias, and T.~Suzuki.
\newblock Dimension-free convergence rates for gradient {Langevin} dynamics in
  {RKHS}.
\newblock \emph{arXiv preprint 2003.00306}, 2020.

\bibitem[Nitanda \& Suzuki(2017)Nitanda and Suzuki]{nitanda2017prticle}
A.~Nitanda and T.~Suzuki.
\newblock Stochastic particle gradient descent for infinite ensembles.
\newblock \emph{arXiv preprint arXiv:1712.05438}, 2017.

\bibitem[Raginsky et~al.(2017)Raginsky, Rakhlin, and
  Telgarsky]{Raginsky_Rakhlin_Telgarsky2017}
M.~Raginsky, A.~Rakhlin, and M.~Telgarsky.
\newblock Non-convex learning via {Stochastic Gradient {Langevin} Dynamics}: a
  nonasymptotic analysis.
\newblock \emph{arXiv e-prints}, pp.\  arXiv:1702.03849, 2017.

\bibitem[Rotskoff \& Vanden-Eijnden(2018)Rotskoff and
  Vanden-Eijnden]{NIPS:Rotskoff&Vanden-Eijnden:2018}
G.~Rotskoff and E.~Vanden-Eijnden.
\newblock Parameters as interacting particles: long time convergence and
  asymptotic error scaling of neural networks.
\newblock In \emph{Advances in Neural Information Processing Systems 31}, pp.\
  7146--7155. Curran Associates, Inc., 2018.

\bibitem[Rotskoff \& Vanden-Eijnden(2019)Rotskoff and
  Vanden-Eijnden]{rotskoff2019trainability}
G.~M. Rotskoff and E.~Vanden-Eijnden.
\newblock Trainability and accuracy of neural networks: An interacting particle
  system approach.
\newblock \emph{arXiv preprint arXiv:1805.00915}, 2019.

\bibitem[Rudin(1987)]{rudin1987real}
W.~Rudin.
\newblock \emph{Real and Complex Analysis (third edition)}.
\newblock Mathematics series. McGraw-Hill, 1987.

\bibitem[Schmidt-Hieber(2020)]{AoS:Schmidt-Hieber:2020}
J.~Schmidt-Hieber.
\newblock Nonparametric regression using deep neural networks with {ReLU}
  activation function.
\newblock \emph{The Annals of Statistics}, 48\penalty0 (4), 2020.

\bibitem[Shardlow(1999)]{Shardlow:1999}
T.~Shardlow.
\newblock Geometric ergodicity for stochastic {PDE}s.
\newblock \emph{Stochastic Analysis and Applications}, 17\penalty0
  (5):\penalty0 857--869, 1999.

\bibitem[Sowers(1992)]{Sowers:1992}
R.~Sowers.
\newblock Large deviations for the invariant measure of a reaction-diffusion
  equation with non-{G}aussian perturbations.
\newblock \emph{Probab. Theory Related Fields}, 92\penalty0 (3):\penalty0
  393--421, 1992.
\newblock ISSN 0178-8051.

\bibitem[Suzuki(2019)]{suzuki2018adaptivity}
T.~Suzuki.
\newblock Adaptivity of deep {ReLU} network for learning in {Besov} and mixed
  smooth {Besov} spaces: optimal rate and curse of dimensionality.
\newblock In \emph{International Conference on Learning Representations}, 2019.
\newblock URL \url{https://openreview.net/forum?id=H1ebTsActm}.

\bibitem[Suzuki(2020)]{Suzuki:NIPS:2020}
T.~Suzuki.
\newblock Generalization bound of globally optimal non-convex neural network
  training: Transportation map estimation by infinite dimensional langevin
  dynamics.
\newblock In \emph{Advances in Neural Information Processing Systems 33}, pp.\
  to appear. Curran Associates, Inc., 2020.

\bibitem[Suzuki \& Nitanda(2019)Suzuki and Nitanda]{suzuki2019deepIntrinsicDim}
T.~Suzuki and A.~Nitanda.
\newblock Deep learning is adaptive to intrinsic dimensionality of model
  smoothness in anisotropic {Besov} space.
\newblock \emph{arXiv preprint arXiv:1910.12799}, 2019.

\bibitem[Welling \& Teh(2011)Welling and Teh]{Welling_Teh11}
M.~Welling and Y.-W. Teh.
\newblock Bayesian learning via stochastic gradient {L}angevin dynamics.
\newblock In \emph{{ICML}}, pp.\  681--688, 2011.

\bibitem[Woodworth et~al.(2020)Woodworth, Gunasekar, Lee, Moroshko, Savarese,
  Golan, Soudry, and Srebro]{pmlr-v125-woodworth20a}
B.~Woodworth, S.~Gunasekar, J.~D. Lee, E.~Moroshko, P.~Savarese, I.~Golan,
  D.~Soudry, and N.~Srebro.
\newblock Kernel and rich regimes in overparametrized models.
\newblock volume 125 of \emph{Proceedings of Machine Learning Research}, pp.\
  3635--3673. PMLR, 09--12 Jul 2020.

\bibitem[Xu et~al.(2018)Xu, Chen, Zou, and Gu]{Xu_Chen_Zou_Gu18}
P.~Xu, J.~Chen, D.~Zou, and Q.~Gu.
\newblock Global convergence of langevin dynamics based algorithms for
  nonconvex optimization.
\newblock In \emph{Advances in Neural Information Processing Systems},
  volume~31, pp.\  3122--3133. Curran Associates, Inc., 2018.

\bibitem[Yehudai \& Shamir(2019)Yehudai and Shamir]{NIPS2019_8886}
G.~Yehudai and O.~Shamir.
\newblock On the power and limitations of random features for understanding
  neural networks.
\newblock In \emph{Advances in Neural Information Processing Systems 32}, pp.\
  6598--6608. Curran Associates, Inc., 2019.

\bibitem[Zhang et~al.(2002)Zhang, Wong, and Zheng]{zhang2002wavelet}
S.~Zhang, M.-Y. Wong, and Z.~Zheng.
\newblock Wavelet threshold estimation of a regression function with random
  design.
\newblock \emph{Journal of Multivariate Analysis}, 80\penalty0 (2):\penalty0
  256--284, 2002.

\bibitem[Zou et~al.(2020)Zou, Cao, Zhou, and Gu]{zou2020gradient}
D.~Zou, Y.~Cao, D.~Zhou, and Q.~Gu.
\newblock Gradient descent optimizes over-parameterized deep {ReLU} networks.
\newblock \emph{Machine Learning}, 109\penalty0 (3):\penalty0 467--492, 2020.

\end{thebibliography}
\bibliographystyle{iclr2021_conference_mod}

\appendix


\section{Proof of  Theorem \ref{thm:LinearLowerMinimaxRate}}\label{sec:ProofMinimaxLowerLinear}

We basically combine the ``convex hull argument'' and the minimax optimal rate analysis for linear estimators developed by \cite{zhang2002wavelet}. 

\cite{zhang2002wavelet} essentially showed the following statement in their Theorem 1.
%

\begin{Proposition}[Theorem 1 of \cite{zhang2002wavelet}]\label{prop:LinearMinimaxRate}
Let $\mu$ be the Lebesgue measure.
Suppose that the space $\Omega$ has even partition $\calA$
such that $|\calA| = 2^{K}$ for an integer $K \in \Natural$, each $A$ has equivalent measure $\mu(A) = 2^{-K}$ for all $A \in \calA$,  and 
$\calA$ is indeed a partition of $\Omega$, i.e., $\cup_{A \in \calA} = \Omega$, $A \cap A' = \emptyset$ for $A,A' \in \Omega$ and $A \neq A'$.
Then, if $K$ is chosen as $n^{-\gamma_1} \leq 2^{-K} \leq n^{-\gamma_2}$ for constants $\gamma_1,\gamma_2 > 0$ that are independent of $n$, then there exists an event $\calE$ such that, for a constant $C' > 0$,
\begin{align*}
P(\calE) \geq 1 + o(1)~\text{and}~|\{x_i \mid x_i \in A~(i\in \{1,\dots,n\})\}| \leq C' n/2^K~~(\forall A \in \calA).
\end{align*}
Moreover, suppose that, for a class $\calF^\circ$ of functions on $\Omega$,  there exists $\Delta > 0$ that satisfies the following conditions:
\begin{enumerate}
\item There exists $F >0$ such that, for any $A \in \calA$, there exists $g \in \calF^\circ$ that satisfies $g(x) \geq \frac{1}{2} \Delta F$ for all $x \in A$,
\item There exists $K'$ and $C''>0$ such that $\frac{1}{n}\sum_{i=1}^n g(x_i)^2 \leq C'' \Delta^2 2^{- K'}$ for any $g \in \calF^\circ$ on the event $\calE$.
\end{enumerate}
Then, 
there exists a constant $F_1$ such that at least one of the following inequalities holds:
\begin{subequations}
\label{eq:SufficientLinearMinimax}
\begin{align}
& \frac{F^2}{4 F_1 C''}  \frac{2^{K'}}{n}\leq  \Rlin(\calF^\circ), 
\label{eq:SufficientLinearMinimax1} \\ 
&\frac{ F^3}{32} \Delta^2 2^{-K} \leq \Rlin(\calF^\circ), 
\label{eq:SufficientLinearMinimax2} 
\end{align}
for sufficiently large $n$. 
\end{subequations}
\end{Proposition}

Before we show the main assertion, we prepare some additional lemmas.
For a sigmoid function $\sigma$, 
let 
$\tilde{\calF}_{C,\tau}^{(\sigma)} := \{ x \in \Real^d \mapsto a \sigma(\tau(w^\top  x + b))) \mid |a| \leq 2 C,~\|w\|\leq 1,~|b| \leq 2~(a,b \in \Real,~w \in \Real^d)\}$ for $C > 0$,~$\tau >0$.

\begin{Lemma}\label{lemm:GaussApprox}
Let $\psi(x) = \frac{1}{2}(\sigma(x + 1) - \sigma(x - 1))$ and $\hat{\psi}$ be its Fourier transform: 
$\hat{\psi}(\omega) := (2\pi)^{-1} \int e^{-\mathrm{i} \omega x} \psi(x) \dd x$.
Let $h > 0$ and $D_w > 0$.
Then,
by setting $\tau = h^{-1} (2\sqrt{d} + 1)D_w$ and $C = \frac{ (2\sqrt{d} + 1)D_w}{\pi h |\hat{\psi}(1)|}$, 
the Gaussian RBF kernel can be approximated by 
\begin{align*}
& \inf_{\check{g} \in \convbar(\tilde{\calF}^{(\sigma)}_{C,\tau})} \sup_{x \in [0,1]^d} 
\left|\check{g}(x) - \exp\left(- \frac{\|x - c\|^2}{2h^2}\right)\right|   \\
& \leq 
\frac{4}{|2\pi \hat{\psi}(1)|} 
\left[ C_dD_w^{2(d-2)}\exp(-D_w^2/2) + \exp(- D_w) \right]
\end{align*}
for any $c \in [0,1]^d$, where $C_d$ is a constant depending only on $d$.
In particular, the right hand side is $O(\exp(-n^\kappa))$ if $D_w = n^\kappa$.
\end{Lemma}

\begin{proof}[Proof of Lemma \ref{lemm:GaussApprox}]

Let $\psi_h(x) = \psi(h^{-1}x)$. 
Suppose that, for $f \in L_1(\Real^d)$, its Fourier transform $\hat{f}(\omega) = (2\pi)^{-d} \int e^{-\mathrm{i} \omega^\top x} f(x) \dd x~(\omega \in \Real^d)$ gives 
$$
\int_{\Real^d} \exp(\mathrm{i} w^\top x)  \hat{f}(w) \dd w = f(x),
$$
for every $x \in \Real^d$\footnote{If $\hat{f}$ is integrable, this inversion formula holds for almost every $x \in \Real^d$ \citep{rudin1987real}. However, we assume a stronger condition that it holds for every $x \in \Real^d$.}.
Then the Irie-Miyake itegral representation (\cite{ICNN:IrieMiyake:1988}; see also the proof of Theorem 3.1 in \cite{ NN:HORNIK:1990551}) gives 
$$
f(x) = \int_{a \in \Real^d} \int_{b \in \Real} \psi(a^\top x + b) \dd \nu(a, b)~~~(\text{a.e.}),
$$
where $\dd \nu(a,b)$ is given by 
$$
\dd \nu(a,b) = \mathrm{Re}\left(\frac{|\omega|^d e^{-\mathrm{i} w b }}{2\pi \hat{\psi}(\omega)} \right) \hat{f}(w a) \dd a \dd b 
$$
for any $\omega \neq 0$.
Since the characteristic function of the multivariate normal distribution gives that 
$$
\int_{\Real^d} \exp(\mathrm{i} w^\top (x-c))  \underbrace{\sqrt{\frac{h^{2d}}{(2\pi)^d}} \exp\left(-\frac{h^2 \|w\|^2}{2}\right)}_{=\hat{f}(w)} \dd w 
= \exp\left( -\frac{\|x -c\|^2}{2 h^2} \right) =:f(x)~~(\forall x \in \Real^d),
$$
we have that 
\begin{align*}
&  \exp\left( -\frac{\|x -c\|^2}{2 h^2} \right) = \\
& \int_{a \in \Real^d} \int_{b \in \Real}  \psi_h(a^\top (x-c) + b)  
 \mathrm{Re}\left(\frac{e^{-\mathrm{i} w b }}{2\pi \hat{\psi}_h(\omega)} \right) 
\sqrt{\frac{|\omega h|^{2d}}{(2\pi)^d}} \exp\left(-\frac{ (\omega h)^2 \|a \|^2}{2}\right)  \dd a \dd b,
\end{align*}
for all $x \in \Real^d$.
Since $\psi_h(\cdot) = \psi(h^{-1} \cdot)$ and $\hat{\psi}_h(\cdot) = h \hat{\psi}(h \cdot)$ by its definition, 
the right hand side is equivalent to 
\begin{align*}
 \int_{a \in \Real^d} \int_{b \in \Real}  \psi(h^{-1}[a^\top (x-c) + b])  
 \mathrm{Re}\left(\frac{e^{-\mathrm{i} w b }}{2\pi h \hat{\psi}(h\omega)} \right) 
\sqrt{\frac{|\omega h|^{2d}}{(2\pi)^d}} \exp\left(-\frac{ (\omega h)^2 \|a \|^2}{2}\right)  \dd a \dd b.
\end{align*}
Here, we set $\omega = h^{-1}$. 
Let $N_{\sigma^2}$ be the probability measure corresponding to the multivariate normal with mean $0$ and covariance $\sigma^2 \Id$, and let $A_D := \{w \in \Real^d \mid \|w\| \leq D \}$.
Let $D_a > 0$ and $D_b = D_a (\sqrt{2 d } + 1)$, and define
\begin{align*}
f_{D_a}(x) :=  
\frac{1}{2D_b N_1(A_{D_a})}
&  \int_{\|a\| \leq D_a,|b| \leq D_b }  \psi(h^{-1}[a^\top (x-c) + b])  
\mathrm{Re}\left(\frac{e^{-\mathrm{i} b/h }}{2\pi h \hat{\psi}(1)} \right)  \times \\
& ~~~ \sqrt{\frac{1}{(2\pi)^d}} \exp\left(-\frac{ \|a \|^2}{2}\right)  \dd a \dd b.
\end{align*}
Then, we can see that, for any $x \in [0,1]^d$, it holds that 
\begin{align*}
& \left|\frac{1}{2D_b N_1(A_{D_a})} f(x) - f_{D_a}(x)\right| \\
& \leq \frac{1}{2D_b N_1(A_{D_a}) |2\pi h \hat{\psi}(1)|} \left[N_1(A_{D_a}^c) \int 2 \exp(-h^{-1}|x|) \dd x + \int_{|b| > D_b} 2 \exp(-[h^{-1}(|b| - 2\sqrt{d}D_a)]) \dd b\right]  \\
& \leq 
\frac{1}{2D_b N_1(A_{D_a})  |2\pi h \hat{\psi}(1)|} 
\left[4h N_1(A_{D_a}^c) + 4 h \exp(-D_a) \right] \\
& =\frac{4h}{2D_b N_1(A_{D_a}) |2\pi h \hat{\psi}(1)|} 
\left[ C_d D_a^{2(d-2)}\exp(- D_a^2/2) + \exp(- D_a) \right],
\end{align*}
where $C_d > 0$ is a constant depending on only $d$, and we used $|a^\top (x-c) + b| \geq |b| - |a^\top (x-c)| \geq |b| - 2 \sqrt{d} D_a$ and $\psi(x) \leq 2 \exp(-|x|)$.
Note that if $D_a = n^{\kappa}$, then the right hand side is $O(h \exp(- n^\kappa))$.
Therefore, since $N_1(A_{D_a}) \leq 1$, by setting $\tau = h^{-1} D_b$, $C = \frac{D_b}{\pi h |\hat{\psi}(1)|}$, we have that 
\begin{align*}
& \inf_{\check{g} \in \convbar(\tilde{\calF}^{(\sigma)}_{C,\tau})} \sup_{x \in [0,1]^d} 
\left|\check{g}(x) - \exp\left(- \frac{\|x - c\|^2}{2h^2}\right)\right|   \\
& \leq 
\frac{4}{|2\pi  \hat{\psi}(1)|} 
\left[ C_d D_a^{2(d-2)}\exp(- D_a^2/2) + \exp(- D_a) \right].
\end{align*}
Hence, by rewriting $D_w \leftarrow D_a$, we obtain the assertion.
As noted above, the right hand is $O(\exp(-n^\kappa))$ if $D_a = n^\kappa$.
\end{proof}

\begin{proof}[Proof of Theorem \ref{thm:LinearLowerMinimaxRate}]
For a sample size $n$, we fix $\mn$ which will be determined later and use Proposition \ref{prop:LinearMinimaxRate} with $\calF^\circ = \calF_\gamma^{(n)}$.
If $\wm{2,\mn} = b \sqrt{\mu_{\mn}^\gamma /2}$ with $|b| \leq 1$ and 
$\wm{1,m} = \mu_{\mn}^{\gamma/2}[u;-u^\top c]/(\sqrt{2(d+1)})$
for $u \in \Real^d$ such that $\|u\| \leq 1$ and $c \in [0,1]^d$, 
then $\|(\wm{1,\mn},\wm{2,\mn})\|^2 \leq \mu_{\mn}^\gamma (1/2 + (1 + |u^\top c|^2)/2(d+1))
\leq \mu_{\mn}^\gamma$. 
Therefore, $\tilde{\varphi}_{u,c}(x) = 
a_\mn \wmb{2,\mn} \sigma_{\mn}(\wm{1,\mn}^\top [x;1])
=
\mu_\mn^{\aone} \overline{(b \mu_{\mn}^{\gamma/2} /\sqrt{2})} \mu_\mn^{s\atwo} \sigma\left(\mu_\mn^{-\atwo + \gamma/2} 
u^\top (x - c)/\sqrt{2(d+1)}\right) \in \calF_\gamma^{(n)} \subset \calF_\gamma$
for all $b\in \Real$ with $|b| \leq 1$, $u \in \Real^d$ with $\|u\| \leq 1$, and $c \in [0,1]^d$.
In other words, $\mu_{\mn}^{\aone + \gamma/2+ s\atwo}(2 C)^{-1} \calF^{(\sigma)}_{C,\tau} \subset \calF_\gamma^{(n)}$ for any $C > 0$ and $\tau = \frac{1}{\sqrt{2(d+1)}}\mu_{\mn}^{-\atwo + \gamma/2}$.

Therefore, by setting $C = (\sqrt{2d}+1)D_w/(\pi h |\hat{\psi}(1)|)$ for $D_w > 0$,
Lemma \ref{lemm:GaussApprox} yields that for any $c \in [0,1]^d$ and given $h > 0$, 
there exists 
$
g \in \convbar(\calF_\gamma^{(n)})
$
such that 
\begin{align*}
& \left\| \mu_{\mn}^{\aone + \gamma/2+ s\atwo}\left(\frac{2 (\sqrt{2d} + 1)D_w}{\pi h |\hat{\psi}(1)|}\right)^{-1} \exp\left( - \frac{\|\cdot - c\|^2}{2 h^2}\right) - g \right\|_\infty \\
& \leq
\mu_{\mn}^{\aone + \gamma/2+ s\atwo} \left(\frac{2 (\sqrt{2d} + 1)D_w}{\pi h |\hat{\psi}(1)|}\right)^{-1} 
\frac{4}{|2\pi \hat{\psi}(1)|} 
\left[ C_d D_w^{2(d-2)}\exp(- D_w^2/2) + \exp(- D_w) \right] \\
& = \mu_{\mn}^{\aone + \gamma/2+ s\atwo} 
\frac{h}{(\sqrt{2d} + 1)D_w} 
\left[ C_d D_a^{2(d-2)}\exp(- D_w^2/2) + \exp(- D_w) \right].
\end{align*}
We let $D_w = n^{\kappa}$ for any $\kappa > 0$ 
and choose $\mu_{\mn}$ 
as $\tau \simeq \mu_{\mn}^{-\atwo + \gamma/2} = D_w h^{-1} = h^{-1}n^{\kappa}$.
We write $\Delta := \mu_{\mn}^{\aone + \gamma/2+ s\atwo} (2C)^{-1} \simeq h^{ \frac{\aone + s \atwo + \gamma/2}{\atwo - \gamma/2} + 1} n^{- \kappa (\frac{\aone + s \atwo + \gamma/2}{\atwo - \gamma/2}+1)}$.
Then, 
it  holds that 
\begin{align}
\left\| \Delta \exp\left( - \frac{\|\cdot - c\|^2}{2 h^2}\right) - g \right\|_\infty
\lesssim \Delta \exp(-n^\kappa). 
\label{eq:DeltaGaussGDiff}
\end{align}

Here, we set $h$ as $h = 2^{-k}$ with a positive integer $k$.
Accordingly, we define a partition $\calA$ of $\Omega$ so that any $A \in \calA$ can be represented as 
$A = [2^{-k} j_1, 2^{-k}(j_1 + 1)] \times \dots \times [2^{-k} j_d, 2^{-k}(j_d + 1)]$ by non-negative integers $0\leq j_i \leq 2^k -1$ ($i=1,\dots,d$).
Note that $|\calA| = 2^{dk} = h^{-d}$. 

For each $A \in \calA$, we define $c_A$ as $c_{A} = (2^{-k}(j_1 +1/2),\dots,2^{-k}(j_d +1/2))^\top$
where $(j_1,\dots,j_d)$ is a set of indexes that satisfies $A = [2^{-k} j_1, 2^{-k}(j_1 + 1)] \times \dots \times [2^{-k} j_d, 2^{-k}(j_d + 1)]$.
For each $A \in \calA$, we define $g_A \in \convbar(\calF_\gamma^{(n)})$ as a function that satisfies \Eqref{eq:DeltaGaussGDiff} for $c = c_A$.

Now, we apply Proposition \ref{prop:LinearMinimaxRate} with $\calF^\circ =\convbar(\calF_\gamma^{(n)})$ and $K = K' = dk$.
Let $R^* := \Rlin(\convbar(\calF_\gamma^{(n)}))$.
First, we can see that there exits a constant $F>0$ such that 
$$
g_A(x) \geq F \Delta~~(\forall x \in A),
$$
where we used $\exp(-n^\kappa) \ll 1$.

Second, in the event $\calE$ introduced in the statement of Proposition \ref{prop:LinearMinimaxRate}, there exists $C$ such that 
$|\{ i  \in \{1,\dots,n\} \mid x_i \in A' \}| \leq C n/2^{-dk}$ for all $A' \in \calA$. In this case, we can check that 
$$
\frac{1}{n} \sum_{i=1}^n \left[ \Delta \exp\left(- \frac{\|x_i - c_A\|^2}{2 h^2} \right) \right]^2 \lesssim 
\Delta^2 h^{d} = \Delta^2 2^{-kd},
$$
by the uniform continuity of the Gaussian RBF.
Therefore, we also have 
\begin{align*}
\frac{1}{n} \sum_{i=1}^n g_A(x_i)^2 
& \leq \frac{2}{n} \sum_{i=1}^n \left[ \Delta \exp\left(- \frac{\|x_i - c_A\|^2}{2 h^2} \right) \right]^2 + c \Delta^2 \exp(-2 n^\kappa)  \\
& \lesssim \Delta^2 (h^{d} + \exp(-2 n^\kappa)),
\end{align*}
where $c > 0$ is a constant.
Thus, as long as $h$ is polynomial to $n$ like $h =\Theta(n^{-a})$, the right hand side is $O(\Delta^2 h^{d})$.

Now, if we write 
$$
\tilde{\beta} = \frac{\aone + s \atwo + \gamma/2}{\atwo - \gamma/2} + 1
= \frac{\aone + (s + 1)\atwo}{\atwo - \gamma/2},
$$ then
we have 
$\Delta \simeq 
h^{ \tilde{\beta} }n^{- \kappa \tilde{\beta}}
$
by its definition.

Here, we choose $k$ as a maximum integer that satisfies 
$
\frac{F^3}{32} \Delta^2 2^{-dk} > R^*.
$
In this situation, it holds that 
$$
h^{2 \tilde{\beta}+ d} n^{-2\kappa \tilde{\beta}} \simeq R^*.
$$
Since \Eqref{eq:SufficientLinearMinimax2} is not satisfied, \Eqref{eq:SufficientLinearMinimax1} must hold, and hence we have
\begin{align*}
& n^{-1} h^{-d} \lesssim  R^* \simeq h^{2 \tilde{\beta} +d } n^{-2\kappa \tilde{\beta}} \\
\Rightarrow ~~~&  h \simeq n^{- \frac{1-2\kappa \tilde{\beta}}{2\tilde{\beta}+ 2d}}.
\end{align*}
Therefore, we obtain that 
\begin{align*}
R^* & \gtrsim n^{- \frac{2\tilde{\beta} + d}{2\tilde{\beta} + 2d}} n^{-\frac{2 \kappa d \tilde{\beta}}{2\tilde{\beta} + 2d}}  \\
& \geq n^{- \frac{2\tilde{\beta} + d}{2\tilde{\beta} + 2 d}} n^{-\kappa'},
\end{align*}
by setting $\kappa' =  \kappa\frac{ 2d \tilde{\beta}}{2\tilde{\beta} + 2d}$.
This gives the assertion.
\end{proof}

\section{Proofs of Proposition \ref{prop:WeakConvergence}, Theorem \ref{thm:ExcessRiskConvRate} and Corollary \ref{cor:RefinedExcessBound}}
\label{sec:ProofConvDeep}


Proposition \ref{prop:WeakConvergence}, Theorem \ref{thm:ExcessRiskConvRate} and Corollary \ref{cor:RefinedExcessBound} can be shown by using Propositions \ref{prop:WeakConvergenceSuzuki} and \ref{thm:ExcessRiskConvRateSuzuki} given in Appendix \ref{sec:AuxLemmas} shown below.

Let $T^\alpha W = (\mu_m^{\alpha}\wm{1,m},\mu_m^{\alpha}\wm{2,m})_{m=1}^\infty$ for $W = (\wm{1,m},\wm{2,m})_{m=1}^\infty$ for $\alpha > 0$, and let us consider a model $h_{W} := f_{T^{-\alpha/2} W}$.
Then, the training error can be rewritten as 
$$
\calLhat(f_W) = \calLhat(h_{T^{\alpha/2}W}).
$$
For notational simplicity, we let $\calLhat(W) := \calLhat(f_W)$. 

Let $\calH^{(M)}$ be $\{W^{(M)} = (\wm{1,m},\wm{2,m})_{m=1}^M \mid \wm{1,m} \in \Real^{d+1},~\wm{2,m}\in \Real,~1\leq m \leq M \}$ and $\iota: \calH^{(M)} \to \calH$ be the zero padding of $W^{(M)}$, that is, 
$\iota(W^{(M)}) = (\wm{1,m}',\wm{2,m}')_{m=1}^\infty \in \calH$ satisfies 
$\wm{1,m}' = \wm{1,m},~\wm{2,m}'=\wm{2,m}~(m \leq M)$ and $\wm{1,m}' = 0,~\wm{2,m}'=0~(m > M)$.
Moreover, we define $\iota^*:\calH \to \calH^{(M)}$ as the map that extracts first $M$ components.
By abuse of notation, we write $f_{W^{(M)}}$ for $W^{(M)} \in \calH^{(M)}$ to indicate $f_{\iota(W^{(M)})}$.
Finally, let $A^{(M)}: \calH^{(M)} \to \calH^{(M)}$ be a linear operator such that $A^{(M)} W^{(M)} = \iota^*( A \iota(W^{(M)}))$,
which is just a truncation of $A$.
Similarly, let $T_M^a W^{(M)}$ for $W^{(M)} \in \calH^{(M)}$ be the operator corresponding to $T^a W$ for $W \in \calH$,
i.e., 
$T_M^a W^{(M)} = \iota^*(T^a \iota(W^{(M)}))$.

\subsection{Auxiliary lemmas}\label{sec:AuxLemmas}

First, we show some key propositions to show the main results.
To do so, we utilize the result by \cite{muzellec2020dimensionfree} and \cite{Suzuki:NIPS:2020}.

 \begin{Assumption}\label{ass:IGLDConvCond}~
{
\begin{assumenum}
\item \label{assum:eigenvalue_cvg}
There exists a constant $c_\mu$ such that $\mu_m \leq c_\mu m^{-2}$.
\item \label{assum:smoothness}
There exist $B,U > 0$ such that 
the following two inequalities hold for some $a \in (1/4,1)$ almost surely:   
\begin{align*}   
&\norm{\nabla  \calLhat(W)}_{\cH} \leq B~(\forall W \in \cH), \\
& \norm{\nabla \calLhat(W) - \nabla \calLhat(W')}_{\cH} \leq L \norm{W - W'}_{\calH_{-a}} ~(\forall W, W' \in \cH).
\end{align*}
\item 
\label{assum:C2_boundedness}
For any data $D_n$, $\calLhat$ is three times differentiable.
Let $\nabla^3 \calLhat(W)$ be the third-order derivative of $\calLhat(W)$.
This can be identified with a third-order linear form
and 
$\nabla^3 \calLhat(W) \cdot (h, k)$ denotes the Riesz representor of  $l \in \cH \mapsto \nabla^3 \calLhat(W) \cdot (h, k, l)$.
There exists $\alpha' \in [0, 1), C_{\alpha'} \in (0, \infty)$ such that 
$\forall W, h, k \in \cH,$ 
$
    \norm{\nabla^3 \calLhat(W) \cdot (h, k)}_{\calH_{-\alpha'}} \leq C_{\alpha'} \norm{h}_{\cH}\norm{k}_{\cH},~~\norm{\nabla^3 \calLhat(W) \cdot (h, k)}_{\cH} \leq C_{\alpha'} \norm{h}_{\calH_{\alpha'}}\norm{k}_{\cH}~~\text{(a.s.)}.
$
%
%
\end{assumenum}}
\end{Assumption}
\begin{Remark}
In the analysis of \cite{Brehier16,muzellec2020dimensionfree,Suzuki:NIPS:2020},  
Assumption \ref{ass:IGLDConvCond}-\ref{assum:C2_boundedness} is imposed for any finite dimensional projection $\calL(W^{(M)})$ as a function on $\calH^{(M)})$ for all $M \geq 1$ instead of $\calL(W)$ as a function of $\calH$.
However, the condition on $\calL(W)$ gives a sufficient condition for any finite dimensional projection in our setting.
Thus, we employed the current version.
\end{Remark}
\begin{Assumption}
\label{ass:BernsteinLightTail}
For the loss function $\ell(y,f(x)) = (y - f(x))^2$, the following conditions holds: 
\begin{assumenum}
\item 
There exists $C > 0$ 
such that for any $f _W~(W \in \calH)$, it holds that  
\begin{align*}
\EE_{X,Y}[(\ell(Y,f_W(X)) - \ell(Y,\fstar(X)))^2] \leq C (\calL(f_W) - \calL(\fstar)). 
\end{align*}
\item  $\beta > 0$ is chosen so that, for any $h:\Real^d \to \Real$ and $x \in \supp(P_X)$, it holds that 
\begin{align*}
\textstyle
\EE_{Y|X=x}\big[\exp\big(- \frac{\beta}{n}(\ell(Y,h(x)) - \ell(Y,\fstar(x)))\big)\big] \leq 1.
\end{align*}
\item There exists $L_h > 0$ such that 
$\|\nabla_W \ell(Y,h_W(X)) - \nabla_{W} \ell(Y,h_{W'}(X))\|_{\cH} \leq L_h \|W- W'\|_{\cH}~~(\forall W,W' \in \calH)$ 
almost surely.
\item 
There exists $C_h$ such that $\|h_W - h_{W'}\|_\infty \leq C_h \|W - W'\|_{\cH}~~(W,W'\in \calH)$.
\end{assumenum}
\end{Assumption}

\begin{Proposition}\label{prop:WeakConvergenceSuzuki}
Assume Assumption \ref{ass:IGLDConvCond} holds and $\beta > \eta$. 
Suppose that $\exists \Rbar > 0$, $0 \leq \ell(Y,f_{W}(X)) \leq \Rbar$ for any $W \in \cH$~(a.s.).
Let $\rho = \frac{1}{1 + \lambda\eta/\mu_1}$ and $b = \frac{\mu_1}{\lambda}B + \frac{c_\mu}{\beta \lambda}$. 
Accordingly, let $\bar{b} = \max\{b,1\}$, $\kappa = \bar{b} + 1$ and 
$\bar{V} = 4 \bar{b}/{\scriptstyle (\sqrt{(1+\rho^{1/\eta})/2} - \rho^{1/\eta})}$.
Then, the spectral gap of the dynamics is given by 
\begin{equation}\label{eq:SpectralGap}
\Lambda^*_\eta = \frac{\min\left(\frac{\lambda}{2 \mu_1}, \frac{1}{2} \right)}{4 \log(\kappa (\bar{V} + 1)/(1-\delta)) } \delta
\end{equation}
where $0 < \delta< 1$ is a real number satisfying $\delta = \Omega(\exp(-\Theta(\mathrm{poly}(\lambda^{-1})\beta)))$.
We define $\Lambda^*_0 = \lim_{\eta \to 0}\Lambda^*_\eta$ (i.e., $\bar{V}$ is replaced by $4\bar{b}/(\scriptstyle \sqrt{(1+\exp(-\frac{\lambda}{\mu_1}))/2} - \exp(-\frac{\lambda}{\mu_1}))$). 
We also define $\textstyle C_{W_0} = \kappa [\bar{V} + 1] + \frac{\sqrt{2} (\Rbar + b)}{\sqrt{\delta}}$.
Then, for any $0 < a < 1/4$, the following convergence bound holds for almost sure observation $D_n$:  
for either $L = \calL$ or $L = \calLhat$,
\begin{align} 
&| \EE_{W_k}[L(W_k) | D_n] 
-
\EE_{W\sim \pi_\infty}[L(W)|D_n] |  \\
&  \leq
C_1 \left[
C_{W_0}  \exp(- \Lambda_\eta^* \eta k )   + 
\frac{\sqrt{\beta}}{\Lambda^*_0}\eta^{1/2-a} \right] = \Xi'_k,
\end{align}
where $C_1$ is a constant depending only on $c_\mu,B,L,C_{\alpha'},a,\Rbar$ (independent of $\eta,k,\beta,\lambda$).
\end{Proposition}

\begin{Proposition}\label{thm:ExcessRiskConvRateSuzuki}
Assume that Assumptions \ref{ass:IGLDConvCond} and \ref{ass:BernsteinLightTail} hold.
Let 
$\alphatil := 1/\{2(\alpha+1)\}$ for a given $\alpha > 0$ and $\theta$ be an arbitrary real number satisfying $0 < \theta < 1 - \alphatil$.
Assume that the true function $\ftrue$ can be represented by $h_{\Wstar} = \ftrue$ for $\Wstar \in \calH_{\theta(\alpha + 1)}$.
Then, if $M \geq  \min\left\{\lambda^{\alphatil/2[\theta(\alpha + 1)]} \beta^{1/2[\theta(\alpha + 1)]}, \lambda^{-1/2(\alpha + 1)},n^{1/2[\theta(\alpha + 1)]}\right\}$, 
the expected excess risk is bounded by
\begin{align}
& \EE_{D^n}\left[ \EE_{W_k^{(M)}}[ \calL(h_{T_M^{\alpha/2}W_k^{(M)}}) |D_n] - 
\calL(\ftrue) \right] \notag \\
& \leq  
 C 
\max \big\{  (\lambda\beta)^{\frac{2\alphatil/\theta}{1 + \alphatil/\theta}} n^{-\frac{1}{1 + \alphatil/\theta}},
\lambda^{-\alphatil} \beta^{-1}, \lambda^{\theta}, 1/n \big\} + \Xi_k',
\label{eq:ExpectedLossConvRateSuzuki}
\end{align}
where $C$ is a constant independent of $n,\beta,\lambda,\eta,k$.
\end{Proposition}

\begin{proof}

Repeating the same argument in Proposition \ref{prop:WeakConvergence} and using the same notation,
Proposition \ref{prop:WeakConvergenceSuzuki} gives 
$$
| \EE_{W_k^{(M)}}[\calL(W_k^{(M)}) | D_n] 
-
\EE_{W\sim \pi_\infty^{(M)}}[\calL(W)|D_n] | \leq \Xi_k',
$$
for any $1 \leq M \leq \infty$. Therefore, we just need to bound the following quantity: 
$
\left|
 \EE_{D^n}\left[ \EE_{W^{(M)}\sim \pi_\infty^{(M)}}[ \calL(h_{T_M^{\alpha/2}W^{(M)}}) |D_n] \right]- 
 \calL(\ftrue)\right|$.

We define $\|W^{(M)}\|_{\cH^{(M)}} := \|\iota^*(W^{(M)})\|_{\cH}$ for $W^{(M)} \in \cH^{(M)}$.
For $a > 0$, we define $\cH_{a}^{(M)}$ be the projection of $\cH_{a}$ to the first $M$ components, 
$\cH_{a}^{(M)} = \{ \iota(W) \mid W \in \cH_a\}$, and we define $\|W^{(M)}\|_{\cH_a^{(M)}} := \|\iota^*(W^{(M)})\|_{\cH_a}$
(note that since $\cH_{a}^{(M)}$ is a finite dimensional linear space, it is same as $\calH$ as a set). 
Let $\nu^{(M)}_\beta$ be the Gaussian measure on $\cH^{(M)}$ with mean 0 and covariance $(\beta A^{(M)})^{-1}$,
and $\nutil^{(M)}$ be the Gaussian measure corresponding to the random variable $T_{M}^{\alpha/2} W^{(M)}$
with $W^{(M)} \sim \nu^{(M)}_\beta$.
Let the concentration function be 
\begin{align*}
& \phi_{\beta,\lambda}^{(M)}(\epsilon)  := \inf_{\substack{W \in \cH^{(M)}_{\alpha + 1}: \\ \calL(h_W) - \calL(\ftrue) \leq \epsilon^2}} \beta \lambda \|W\|^2_{\cH^{(M)}_{\alpha + 1}}
- \log \nutil^{(M)}(\{W \in \cH^{(M)} : \|W\|_{\cH^{(M)}} \leq \epsilon \}) + \log(2),
\end{align*}
where, if there does not exist $W \in \cH^{(M)}_{\alpha + 1}$ that satisfies the condition $\inf$, 
then we define $\phi^{(M)}_{\beta,\lambda}(\epsilon) = \infty$,
then 
Let $\epsilonstar > 0$ be 
$$
\epsilonstar := \max\{\inf\{\epsilon > 0 \mid \phi_{\beta,\lambda}(\epsilon) \leq \beta \epsilon^2\}, 1/n\}.
$$
Then, \cite{Suzuki:NIPS:2020} showed the following bound:
\begin{align}\label{eq:ExpectedLossConvRateSuzuki_proof}
& \left| \EE_{D^n}\left[ \EE_{W^{(M)}\sim \pi_\infty^{(M)}}[ \calL(h_{T_{(M)}^{\alpha/2}W^{(M)}}) |D_n] - 
\calL(\ftrue) \right] \right| \notag \\
& 
\leq  
 C \max\Big\{ 
\epsilonstar^2 ,
\big(\tfrac{ \beta }{n} \epsilonstar^2  + n^{-\frac{1}{1+\alphatil/\theta}} (\lambda\beta)^{\frac{2\alphatil/ \theta}{1+\alphatil/\theta}} \big), \frac{1}{n}
\Big\}. 
\end{align}
They also showed that, for $M = \infty$, it holds that   
\begin{align*}
\epsilonstar^2 \lesssim \max\left\{(\lambda \beta)^{-\alphatil} \beta^{-(1-\alphatil)}, \lambda^{\theta},n^{-1}\right\}
=  \max\left\{\lambda^{-\alphatil} \beta^{-1}, \lambda^{\theta},n^{-1}\right\}.
\end{align*}
Substituting this bound of $\epsilonstar$ to \Eqref{eq:ExpectedLossConvRateSuzuki_proof}, we obtain \Eqref{eq:ExpectedLossConvRateSuzuki} for $M = \infty$.
Moreover, in their proof, if $M \geq (\epsilonstar)^{-1/[\theta(\alpha + 1)]}$, then 
$$
 \inf_{\substack{W \in \cH^{(M)}_{\alpha + 1}: \\ \calL(h_W) - \calL(\ftrue) \leq \epsilon^2}} \beta \lambda \|W\|^2_{\cH^{(M)}_{\alpha + 1}} \lesssim \beta (\epsilonstar)^2.
$$
Finally, since $\nutil^{(M)}$ is a marginal distribution of $\nutil^{(\infty)}$, it holds that 
$$
- \log \nutil^{(M)}(\{W \in \cH^{(M)} : \|W\|_{\cH^{(M)}} \leq \epsilon \})  \leq 
- \log \nutil^{(\infty)}(\{W \in \cH : \|W\|_{\cH} \leq \epsilon \}). 
$$
Therefore, as long as $M \geq (\epsilonstar)^{-1/[\theta(\alpha + 1)]}$, the rate of $\epsilonstar$ is not deteriorated from $M = \infty$.
In other words, if $M \geq  \min\left\{\lambda^{\alphatil/2[\theta(\alpha + 1)]} \beta^{1/2[\theta(\alpha + 1)]}, \lambda^{-\theta/2[\theta(\alpha + 1)]},n^{1/2[\theta(\alpha + 1)]}\right\}$, 
the bound \eqref{eq:ExpectedLossConvRateSuzuki} holds.
%
\end{proof}
\begin{Remark}
\cite{Suzuki:NIPS:2020} showed Proposition \ref{thm:ExcessRiskConvRateSuzuki} under a condition $\alpha > 1/2$.
However, this is used only to ensure Assumption \ref{ass:IGLDConvCond}.
In our setting, we can show Assumption \ref{ass:IGLDConvCond} holds directly and thus we may omit the condition $\alpha > 1/2$.
\end{Remark}

\subsection{Proofs of Proposition \ref{prop:WeakConvergence}, Theorem \ref{thm:ExcessRiskConvRate} and Corollary \ref{cor:RefinedExcessBound}}

Here, we give the proofs of Proposition \ref{prop:WeakConvergence} and Theorem \ref{thm:ExcessRiskConvRate} simultaneously.
\begin{proof}[Proof of Proposition \ref{prop:WeakConvergence} and Theorem \ref{thm:ExcessRiskConvRate}]
Let $\Rbar =  (2 \sum_{m=1}^\infty a_m R + U)^2$. Then, we can easily check that $(y_i - f_W(x_i))^2 \leq \Rbar$.
As stated above, we use Propositions \ref{prop:WeakConvergenceSuzuki} and \ref{thm:ExcessRiskConvRateSuzuki} to show the statements.

First, we show Proposition \ref{prop:WeakConvergence} for the dynamics of $W_k^{(M)}$ for any $1 \leq M \leq \infty$.
However, it suffices to show the statement only for $M = \infty$ because the finite dimensional version can be seen as a specific case of the infinite dimensional one.
Actually, the dynamics of $W_k^{(M)}$ is same as that of $\iota(\tilde{W}_k)$ where $\tilde{W}_k \in \calH$ obeys the following dynamics:
$$
\tilde{W}_{k+1} = S_\eta\left(\tilde{W}_{k} - \eta \nabla \calLhat(f_{\iota(\tilde{W}_{k})}) + \sqrt{\frac{2 \eta}{\beta}} \xi_{k} \right).
$$
This is because $f_{\iota(\tilde{W}_k)}$ is determined by only the first $M$ components $\iota(\tilde{W}_{k})$, 
$\iota(\nabla \calLhat(f_{\iota(\tilde{W}_{k})})) = \nabla_{W^{(M)}} \calLhat(f_{W^{(M)}})|_{W^{(M)} = \iota(\tilde{W}_k)}$
and $S_\eta$ is a diagonal operator. 
Since the components of $\tilde{W}_k$ with indexes higher than $M$ does not affect the objective,
smoothness of the objective is not lost.
The stationary distribution $\pi_\infty^{(M)}$ of the continuous dynamics corresponding to $W^{(M)}$ is a probability measure on $\calH^{(M)}$ that satisfies 
$$
\frac{\dd \pi_\infty^{(M)}}{\dd \nu^{(M)}_\beta}(W^{(M)}) \propto \exp(-\beta \calLhat(f_{W^{(M)}})),
$$ 
where $\nu^{(M)}_\beta$ is the Gaussian measure on $\Real^{M\times (d+2)}$ with mean 0 and covariance $(\beta A^{(M)})^{-1}$.
We can notice that this is the marginal distribution of the stationary distribution of the continuous time counterpart of $\tilde{W}_k$: $\dd \tilde{\pi}_\infty(\tilde{W}) \propto \exp(-\beta \calLhat(f_{\iota(\tilde{W})}))\dd \nu_\beta$.
Therefore, we just need to consider an infinite dimensional one. 
For this reasoning, we show the convergence for the original infinite dimensional dynamics $(W_k)_{k=1}^\infty$.
The convergence of the finite dimensional one $(W^{(M)}_k)_{k=1}^\infty$ can be shown by the same manner using the argument above.

To show Proposition \ref{prop:WeakConvergence}, we use Propositions \ref{prop:WeakConvergenceSuzuki}.
To do so, we need to check validity of Assumptions \ref{ass:IGLDConvCond}.
First, we check Assumption \ref{ass:IGLDConvCond}.
Assumption \ref{ass:IGLDConvCond}-\ref{assum:eigenvalue_cvg} is ensured by Assumption \ref{ass:sigmamam}.
Next, we check Assumption \ref{ass:IGLDConvCond}-\ref{assum:smoothness}.
The boundedness of the gradient can be shown as follows:
\begin{align*}
& \| \nabla \calLhat(f_W)\|_{\calH}^2 \\
=  
&\sum_{m=1}^\infty 
\Big( 
\Big\| \frac{1}{n}\sum_{i=1}^n 2 (f_W(x_i) - y_i) \wmb{2,m} a_m [x_i;1] \sigma_m'(\wm{1,m}^\top [x_i;1]) \Big\|^2  \\
& ~~+ 
\Big| \frac{1}{n}\sum_{i=1}^n 2 (f_W(x_i) - y_i) a_m  \tanh'(\wm{2,m}/R) \sigma_m(\wm{1,m}^\top [x_i;1]))_{m=1}^\infty \Big|^2 
\Big) \\
\leq 
& 
\sum_{m=1}^\infty 4 \Rbar R^2 a_m^2 (d+1) C_\sigma^2
+ 4 \Rbar a_m^2 \\
& ~~~~(\because |f_W(x_i) - y_i |\leq \Rbar,~ \|\sigma_m'\|_\infty \leq C_\sigma,~\|\tanh'\|_\infty \leq 1)\\
\leq
& 4 \Rbar [R^2C_\sigma^2 (d+1)+1] \sum_{m=1}^\infty a_m^2 < \infty. 
\end{align*}
Similarly, we can show the Lipschitz continuity of the gradient as 
\begin{align*}
& \| \nabla \calLhat(f_W) - \nabla \calLhat(f_{W'})\|_{\calH}^2 \\
& \leq 
\sum_{m=1}^\infty \mu_m^{-2\aone} \mu_m^{2\aone} 
\Big\{ 4 \Rbar a_m^2 (d+1) C_\sigma^2[(\wm{2,m} - \wm{2,m}')^2 + R^2 \|\wm{1,m} - \wm{1,m}'\|^2] \\
& ~~~~+ 4 \Rbar a_m^2 [(\wm{2,m} - \wm{2,m}')^2/R^2 + C_\sigma^2(d+1) \|\wm{1,m} - \wm{1,m}'\|^2]  
\Big\}~~~~(\because \|\tanh''\|_\infty \leq 1) \\
& \leq 
4 \Rbar [(d+1) C_\sigma^2 (1+R^2) + 1/R^2 + C_\sigma^2(d+1)]
\max_{m \in \Natural}\{ \mu_m^{-2\aone} a_m^2 \} \\
& ~~~\times
\sum_{m=1}^\infty \mu_m^{2\aone} 
[(\wm{2,m} - \wm{2,m}')^2 + \|\wm{1,m} - \wm{1,m}'\|^2] \\
& \lesssim 
\|W - W'\|_{\calH_{-\aone}}^2.
\end{align*}
We can also verify Assumption \ref{ass:IGLDConvCond}-\ref{assum:C2_boundedness} in a similar way.
Then, we have verified Assumption \ref{ass:IGLDConvCond}.
Therefore, we may apply Proposition \ref{prop:WeakConvergenceSuzuki}, and then we obtain Proposition \ref{prop:WeakConvergence}.

Next, we show Theorem \ref{thm:ExcessRiskConvRate} by using Proposition \ref{thm:ExcessRiskConvRateSuzuki}.
For that purpose, we need to we verify Assumption \ref{ass:BernsteinLightTail}.
The first condition can be verified as  
\begin{align*}
& \EE_{X,Y}[((Y - f_W(X))^2 - (Y - \ftrue(X))^2)^2] \\
& = \EE_{X,\epsilon}[((\ftrue(X) + \epsilon - f_W(X))^2 - \epsilon^2)^2] \\
& = \EE_X[( (\ftrue(X) - f_W(X))^2 + 2 \epsilon (\ftrue(X) - f_W(X)))^2] \\
& = \EE_X[(\ftrue(X) - f_W(X))^4 + 2 \epsilon (\ftrue(X) - f_W(X))(\ftrue(X) - f_W(X))^2 + \epsilon^2 (\ftrue(X) - f_W(X))^2] \\
& = \|\ftrue - f_W\|_{\infty}^2\EE_X[(\ftrue(X) - f_W(X))^2]   + U^2 \EE_X[(\ftrue(X) - f_W(X))^2] \\
& \leq  \Rbar \EE_X[(\ftrue(X) - f_W(X))^2]
= \Rbar (\calL(f_W) - \calL(\ftrue)).\\
\end{align*}
The second condition can be checked as follows.
Note that 
\begin{align*}
& \EE_{Y|X=x}\left(\exp\left\{- \frac{\beta}{n}[(Y - f_W(x))^2 - (Y - \ftrue(x))^2]\right\} \right)\\
& = 
\EE_{\epsilon}\left(\exp\left[- \frac{\beta}{n}(\ftrue(x) - f_W(x))^2 - 2 \epsilon (f_W(x) - \ftrue(x))]\right\} \right)\\
& =
\exp\left[- \frac{\beta}{n}(\ftrue(x) - f_W(x))^2\right] \EE_{\epsilon}\left\{\exp\left[\frac{2 \beta}{n}  \epsilon (f_W(x) - \ftrue(x))\right] \right\}\\
& \leq
\exp\left[- \frac{\beta}{n}(\ftrue(x) - f_W(x))^2\right] 
\exp\left[\frac{1}{8}\frac{4 \beta^2}{n^2}  4U^2 (f_W(x) - \ftrue(x))^2 \right].
\end{align*}
Thus, under the condition $\beta \leq n/(2U^2)$, the right hand side can be upper bounded by 
$$
\exp\left[ - \frac{\beta}{n} \left(1 - 2 \frac{U^2 \beta}{n} \right) (f_W(x) - \ftrue(x))^2 \right] 
\leq 1.
$$
Next, we check the third and fourth conditions. 
Noting that 
\begin{align*}
& \nabla_W h_W(X) 
 \\
= & 
\Big( 
a_m \overline{(\mu_m^{-\alpha/2}\wm{2,m})}   \mu_m^{-\alpha/2} [x_i;1] \sigma_m'(\mu_m^{-\alpha/2}\wm{1,m}^\top [x_i;1]),  \\
& 
~~~ a_m \mu_m^{-\alpha/2}  \tanh'(\mu_m^{-\alpha/2} \wm{2,m}/R) \sigma_m(\mu_m^{-\alpha/2} \wm{1,m}^\top [x_i;1]))_{m=1}^\infty  
\Big)_{m=1}^\infty,
\end{align*}
we have that 
\begin{align*}
& \|\nabla_{W}h_{W}(X)\|_{\calH}^2\\
& \leq \sum_{m=1}^\infty  a_m^2 \mu_m^{-\alpha} [(d+1) R^2 C_\sigma^2 +  1] \\
& \leq [(d+1) R^2 C_\sigma^2 +  1]  \sum_{m=1}^\infty \mu_m^{-\alpha + 2\aone} \\
& \leq [(d+1) R^2 C_\sigma^2 +  1]  c_\mu^{-\alpha + 2\aone} \sum_{m=1}^\infty m^{-2(-\alpha + 2\aone)}
=: C_1 < \infty \\
& ~~~~(\because -\alpha + 2\aone = \aone > 1/2),
\end{align*}
and 
\begin{align*}
& \|\nabla_{W}h_{W}(X) - \nabla_{W}h_{W'}(X)\|_{\calH}^2\\
& \leq \sum_{m=1}^\infty  a_m^2 \mu_m^{-\alpha} (d+1) [\mu_m^{-\alpha}(\wm{2,m} -\wm{2,m}')^2  + R^2 \mu_m^{-\alpha}
\|\wm{1,m} - \wm{1,m}'\|^2] \\
& ~~+  a_m^2 \mu_m^{-\alpha} [\mu_m^{-\alpha}(\wm{2,m} - \wm{2,m}')^2/R^2 + C_\sigma^2 (d+1)\mu_m^{-\alpha}\|\wm{1,m} - \wm{1,m}'\|^2] \\
& 
\leq 
\sum_{m=1}^\infty  a_m^2 \mu_m^{-2 \alpha}[(d+1)(1+R^2) + 1/R^2 + C_\sigma^2 (d+1)]
[\|\wm{1,m} - \wm{1,m}'\|^2 + (\wm{2,m} -\wm{2,m}')^2] \\
& \leq 
c_\mu^{2\aone} \max_m\{ \mu_m^{2 (\aone -\alpha)}\} [(d+1)(1+R^2) + 1/R^2 + C_\sigma^2 (d+1)] \|W - W'\|_{\calH}^2
=: C_2 \|W - W'\|_{\calH}^2,
\end{align*}
for a constant $0 < C_2 < \infty$.
Therefore, it holds that 
\begin{align*}
|h_W(X) - h_{W'}(X)|^2 \leq C_1 \|W - W'\|_{\calH}^2, 
\end{align*}
which yields the forth condition, and we also have
\begin{align*}
&\| \nabla_W \ell(Y,h_W(X)) - \nabla_W \ell(Y,h_{W'}(X))\|_{\calH}^2  \\
= &
\|2(h_W(X) - Y )\nabla_W h_W(X) - 2(h_{W'}(X) - Y)\nabla_{W}h_{W'}(X)\|_{\calH}^2 \\
\leq 
& 
2 \|2(h_W(X) - Y ) (\nabla_W h_W(X) - \nabla_W h_W(X))\|_{\calH}^2  \\
& + 2 \|2(h_{W}(X) - h_{W'}(X)) \nabla_{W}h_{W'}(X)\|_{\calH}^2 \\
& \leq 8 \Rbar C_2 \|W - W'\|_{\calH}^2
+ 8 C_1^2  \|W - W'\|_{\calH}^2 \lesssim \|W - W'\|_{\calH}^2,
\end{align*}
which yields the third condition.

Since $\ftrue \in \calF_\gamma$, there exists $\Wstar \in \calH_\gamma$ such that $\ftrue = f_{\Wstar}$.
Therefore, applying Proposition \ref{thm:ExcessRiskConvRateSuzuki} with $\alpha = \aone$ ($\alphatil = 1/[2(\aone + 1)]$) and $\theta = \gamma/(1 + \aone)$ (since $\gamma < 1/2+\aone$, the condition $\theta < 1- \alphatil$ is satisfied),
we obtain that 
for $M \geq  \min\left\{\lambda^{1/4\gamma(\aone + 1)} \beta^{1/2\gamma}, \lambda^{-1/2(\aone + 1)},n^{1/2\gamma}\right\}$, 
the following excess risk bound holds: 
\begin{align*}
\EE_{D^n}\left[ \EE_{W_k^{(M)}}[\calL(W_k^{(M)})|D_n] - \calL(\fstar) \right]
\lesssim 
\max \big\{  (\lambda\beta)^{\frac{2\alphatil/\theta}{1 + \alphatil/\theta}} n^{-\frac{1}{1 + \alphatil/\theta}},
\lambda^{-\alphatil} \beta^{-1}, \lambda^{\theta}, 1/n \big\} + \Xi_k.
\end{align*}
Finally, by noting $\calL(W_k^{(M)}) - \calL(\fstar) = \|f_{W_k^{(M)}} - \fstar \|_{\LPiPx}^2$, we obtain the assertion.
\end{proof}

Finally, we give the proof of Corollary \ref{cor:RefinedExcessBound}.
\begin{proof}[Proof of Corollary \ref{cor:RefinedExcessBound}]
Note that 
\begin{align*}
& f_W(x) \\
& = \sum_{m=1}^\infty a_m \wmb{2,m}  \sigma_m(\wm{1,m}^\top [x;1]) \\
& = \sum_{m=1}^\infty \mu_m^{\aone} \wmb{2,m} \mu_m^{q\atwo} \mu_m^{-q\atwo} \mu_m^{s\atwo}\sigma(\mu_m^{-\atwo}\wm{1,m}^\top [x;1]) ~~~(\because a_m = \mu_m^{\aone},~b_m =\mu_m^{\atwo})\\
& = \sum_{m=1}^\infty \mu_m^{\aone + q\atwo} \wmb{2,m} \mu_m^{-(s -q)\atwo } \sigma(\mu_m^{-\atwo}\wm{1,m}^\top [x;1]).
\end{align*}
Therefore, we may redefine $\aone' \leftarrow \aone + q\atwo$ and $s' \leftarrow s-q$
so that we obtain another representation of the model $\calF_\gamma$:
$$
\calF_\gamma = \left\{f_W(x) = \sum_{m=1}^\infty \mu_m^{\aone'} \wmb{2,m} \check{\sigma}_m(\wm{1,m}^\top [x;1]) ~\Big|~ W \in \calH_\gamma,~\|W\|_{\calH_\gamma} \leq 1 \right\},
$$
where $\check{\sigma}_m(\cdot) = \mu_m^{-s'\atwo } \sigma(\mu_m^{-\atwo} \cdot)$.
Note that the condition $0\leq q \leq s- 3$ gives $s-q \geq 3$. 
Therefore, Assumptions \ref{ass:IGLDConvCond} and \ref{ass:BernsteinLightTail} are valid even for the redefined parameters $\aone'$, $s'$ and $\check{\sigma}_m$ instead of $\aone$, $s$ and $\sigma_m$.
Therefore, we can apply Theorem \ref{thm:ExcessRiskConvRate} by simply replacing $\aone$ by $\aone' = \aone + q\atwo$.
\end{proof}

\end{document}